\documentclass[letterpaper]{article} 
\usepackage[preprint]{aaai2027}  
\usepackage[hyphens]{url}  
\usepackage{graphicx} 
\usepackage{natbib}  
\usepackage{caption} 
\usepackage{algorithm}
\usepackage{algorithmic}

\usepackage{newfloat}
\usepackage{listings}
\DeclareCaptionStyle{ruled}{labelfont=normalfont,labelsep=colon,strut=off} 
\floatstyle{ruled}
\newfloat{listing}{tb}{lst}{}
\floatname{listing}{Listing}

\usepackage{booktabs}

\usepackage{amsmath}   
\usepackage{bm}        
\usepackage{amssymb}
\newcommand{\bikan}{BiKAN}
\newcommand{\ms}[2]{#1{\tiny$\pm$#2}}
\newtheorem{proposition}{Proposition}

\title{BiKAN: Restoring Collapsed Basis of Binary Kolmogorov--Arnold Networks}
\author{
    Kazi Ahmed Asif Fuad\textsuperscript{\rm 1},
    Lizhong Chen\textsuperscript{\rm 1}
}

\affiliations{
    \textsuperscript{\rm 1}Department of EECS\\
    Oregon State University\\
    Corvallis, OR 97331\\
    fuadk@oregonstate.edu, chenliz@oregonstate.edu
}

\begin{document}

\maketitle

\begin{abstract}
Binarizing a polynomial Kolmogorov--Arnold Network (KAN) not only changes parameter precision, but also alters the function space available to each layer. When activations are restricted to ${-1,+1}$, all even powers reduce to $1$ and all odd powers reduce to $x$, causing the elementwise polynomial basis to collapse to constant and first-order responses. We refer to this structural failure as \emph{Spatial Orthogonality Collapse}. Our proposed \bikan{} addresses this critical issue by augmenting each binary KAN layer with selected degree-2 Walsh characters. Fixed circular channel rolls generate pairwise parities, and learned binary projections mix them using the same XNOR--popcount operations as the remaining W1A1 paths. This restores explicit pairwise coordinates without learned routing or multiplier-based feature generation. Experiments on CIFAR-10 confirms that removing parity reduces accuracy by $1.23$ points over five paired seeds ($p=0.003$), the gain increases as width decreases, and accuracy improves monotonically as more parity planes are added. At an equal $\sim$11.9M-parameter budget, parity outperforms conventional widening by $3.09$ points ($p<10^{-4}$). At W1A1, \bikan{} reaches $99.48\%$, $84.38\%$, and $55.81\%$ on MNIST, CIFAR-10, and CIFAR-100, respectively. Post-route Zynq-7020 FPGA results show that the repair remains hardware-efficient; the convolutional design cuts DSP usage from 164 to 72 and estimated compute-core latency from 401 to 54.8 ms, while the power-of-two-aware dense design achieves zero-DSP inference with a 0.03-point accuracy loss. The BiKAN implementation is available at
\url{https://github.com/OSU-STARLAB/BiKAN}.
\end{abstract}



\section{Introduction}
\label{sec:introduction}

Kolmogorov--Arnold Networks (KANs) replace scalar weights with learnable univariate functions on edges, making the basis itself part of the model representation~\citep{liu2024kan}. Binary neural networks pursue the opposite objective: with 1-bit weights and activations (W1A1), dot products become XNOR--popcount operations, sharply reducing arithmetic and storage cost~\citep{rastegari2016xnor}. Combining KAN and binary is attractive for edge deployment, but exposes a KAN-specific failure. Once an activation is restricted to $x\in\{-1,+1\}$,
\begin{equation}
x^{2k}=1,\qquad x^{2k+1}=x ,
\label{eq:intro_collapse}
\end{equation}
so an elementwise polynomial dictionary of any nominal degree collapses to constant and first-order coordinates. A polynomial KAN layer therefore loses the higher-order basis structure it was designed to exploit. We call this \emph{Spatial Orthogonality Collapse}. Unlike ordinary quantization error, this failure is algebraic: no optimizer or surrogate gradient can make $x_i^2$ distinct from $1$ on the Boolean cube.

A generic remedy is to add network capacity. Widening is effective in conventional binary networks~\citep{mishra2018wrpn}, and sufficiently wide binary KANs also recover accuracy. Yet widening is indifferent to \emph{which} coordinates disappeared. In our CIFAR-10 study, a fully trained $4\times$ widened binary KAN uses 94.67M parameters to slightly exceed the 11.94M teacher-width \bikan{}, illustrating that brute-force capacity can work but at a very different model-size operating point. This motivates a more targeted question: if W1A1 changes the domain of the representation, can the missing capacity be restored using a basis native to that domain?

On the Boolean cube, every function admits a Walsh--Fourier expansion
$f(\mathbf{x})=\sum_{S}\widehat f(S)\chi_S(\mathbf{x})$ with
$\chi_S(\mathbf{x})=\prod_{i\in S}x_i$~\citep{odonnell2014boolean}.
Constants and coordinates form degrees $0$ and $1$; the next level contains the degree-2 parities $x_i x_j$. These pairwise coordinates are absent from the collapsed affine span, yet on binary inputs each requires one XNOR. This observation leads to \bikan{}, which restores selected degree-2 Walsh coordinates explicitly. A set of fixed circular channel rolls forms parities $q_c q_{c-r}$, and a learned binary projection mixes the resulting features using the same XNOR--popcount primitive as the remaining W1A1 paths. The rolls are compile-time wiring and require no learned routing. The default construction restores only a structured circulant subset of the full quadratic Walsh basis; accordingly, our claim is not that parity replaces width or depth, but that it exposes useful interactions directly rather than synthesizing them indirectly through additional generic capacity.

The experiments bear out this basis-restoration account. On CIFAR-10, removing parity reduces accuracy by $1.23$ points over five paired seeds ($p=0.003$); the advantage grows to $2.87$ points at quarter width; and accuracy rises monotonically as more parity planes are exposed. Crucially, at an equal $\sim$11.9M-parameter budget, allocating capacity to the parity path outperforms conventional widening by $3.09$ points ($p<10^{-4}$), separating structured feature restoration from a parameter-count effect. At the final W1A1 operating point, \bikan{} reaches $99.48\%$, $84.38\%$, and $55.81\%$ on MNIST, CIFAR-10, and CIFAR-100, compared with $99.53\%$, $83.11\%$, and $53.24\%$ for the paired FP32 teachers, respectively.

The repair also preserves the hardware motivation for binarization. On a Xilinx Zynq-7020, the convolutional W1A1 implementation reduces post-place-and-route DSP usage from 164 to 72 and LUT usage from 25.2k to 9.1k while changing MNIST accuracy by only $-0.06$ points; under the same HLS scheduling model, compute-core latency falls from 401 to 54.8\,ms ($7.3\times$). A dense power-of-two-aware implementation further removes the remaining scale multipliers and reaches a strictly 0-DSP design at 97.61\%. These results show that the added Boolean coordinates do not undo the deployment advantages that motivate W1A1. Our contributions are:
\begin{itemize}
    \item We identify Spatial Orthogonality Collapse in W1A1 polynomial KANs and formalize why the basis collapses to constant and first-order responses.
    \item We introduce a structured Walsh-parity path that restores selected degree-2 coordinates with fixed wiring and XNOR, and validate the mechanism through paired ablation, dose--response, capacity-starvation, synthetic parity, and equal-parameter width controls.
    \item We show that the representation remains compatible with XNOR--popcount inference and demonstrate post-route FPGA resource and compute-kernel latency reductions, including a 0-DSP dense implementation.
\end{itemize}

\section{Related Work}
\label{sec:related_work}

\paragraph{KANs and efficient deployment.}
KANs were introduced with B-spline edge functions~\citep{liu2024kan}, while subsequent variants pursue cheaper or alternative bases, including EfficientKAN, RBF-based FastKAN, and convolutional polynomial KANs~\citep{blealtan2024efficientkan,li2024fastkan,bodner2024convkan,drokin2024kagn}. Recent work addresses deployment more directly. QuantKAN develops QAT/PTQ methods across KAN families~\citep{quantkan}; KANtize studies low-bit spline coefficients and lookup tables~\citep{kantize}; and KANEL\'E combines quantization, pruning, and LUT-based FPGA evaluation~\citep{kanele}. These methods seek to approximate or compress an existing KAN representation. Our starting point is different: at W1A1, distinct polynomial basis functions can become algebraically identical, so the representation itself changes.

BiKA is the closest binary hardware-oriented KAN work~\citep{bika}. It replaces nonlinear KAN functions with learnable binary thresholds and realizes a comparator-and-accumulator architecture. \bikan{} instead asks what basis information is lost when a polynomial KAN is binarized and restores selected missing coordinates explicitly. To our knowledge, prior KAN quantization or hardware work has not formulated W1A1 polynomial-basis collapse or connected its repair to degree-2 Walsh characters.

\paragraph{Binary networks and Boolean structure.}
XNOR-Net established XNOR--popcount inference for binary CNNs~\citep{rastegari2016xnor}; WRPN uses widening to recover capacity~\citep{mishra2018wrpn}, while Bi-Real Net, IR-Net, and ReActNet improve optimization and information flow through shortcuts, surrogate gradients, and learnable activation reshaping~\citep{liu2018bireal,qin2020irnet,liu2020reactnet}. \bikan{} adopts training machinery but targets a different failure: binarization collapses distinct KAN basis functions rather than reducing numerical precision.

Walsh--Hadamard transforms have also appeared as efficient structured linear maps in neural networks~\citep{le2013fastfood,sindhwani2015structured}. In \bikan{}, however, the Hadamard path remains a linear remapping and does not itself expose $x_i x_j$ as a direct feature; the proposed parity path constructs that degree-2 Walsh character explicitly. This distinction is also hardware-aligned: parity maps to fixed wiring plus XNOR, complementing FPGA frameworks such as FINN and Boolean-logic synthesis approaches such as LogicNets~\citep{umuroglu2017finn,logicnets}.


\section{What Binarization Destroys in a KAN}
\label{sec:analysis}

\paragraph{Elementwise basis collapse.}
Let $\mathbf{x}\in\{-1,+1\}^{n}$ and consider a KAN layer whose edge
functions are univariate polynomials, including Gram- or Chebyshev-based
parameterizations. On the Boolean domain,
\begin{equation}
    x_i^{2k}=1,\qquad x_i^{2k+1}=x_i,
    \label{eq:collapse}
\end{equation}
so every polynomial of one binary coordinate reduces to $a_i+b_i x_i$.
Because a KAN layer adds these univariate edge responses, its elementwise
polynomial basis collapses, before the next nonlinearity, to
$\operatorname{span}\{1,x_1,\ldots,x_n\}$. W1A1 therefore removes the
higher-order coordinates of the original basis; the loss is structural, not
merely quantization noise.

\paragraph{The Walsh--Fourier view.}
Every $f:\{-1,+1\}^{n}\!\to\!\mathbb{R}$ has the unique expansion
\begin{equation}
    f(\mathbf{x})
    =\sum_{S\subseteq[n]} \widehat f(S)\,\chi_S(\mathbf{x}),
    \qquad
    \chi_S(\mathbf{x})=\prod_{i\in S}x_i,
\end{equation}
where the Walsh characters are orthonormal under the uniform
measure~\citep{odonnell2014boolean}. Degree-2 characters are exactly
$\chi_{\{i,j\}}(\mathbf{x})=x_i x_j$. A Hadamard projection is linear and
therefore exposes no explicit degree-2 coordinate; a later sign may induce
higher-order Fourier content only indirectly. Width and depth can likewise
synthesize such interactions through threshold composition, whereas on the
Boolean cube $x_i x_j$ is available directly as one parity, i.e., one XNOR
under bit encoding.

\begin{proposition}
\label{prop:parity}
For $n\geq2$, let
\begin{equation}
    \Phi_2(\mathbf{x})
    =\bigl[1,\;x_1,\ldots,x_n,\;\{x_i x_j\}_{1\leq i<j\leq n}\bigr].
\end{equation}
\textnormal{(i)} Linear functions of $\Phi_2$ are exactly the pseudo-Boolean
functions of Walsh degree at most $2$. \textnormal{(ii)} For $i\neq j$,
$x_i x_j$ is not affine in $\mathbf{x}$, and no single linear-threshold unit
$\operatorname{sign}(\mathbf{a}^{\top}\mathbf{x}+b)$ computes it on the
Boolean cube.
\end{proposition}

\begin{quote}
\emph{Proof sketch.}
The entries of $\Phi_2$ are precisely the Walsh characters with
$|S|\leq2$, so orthogonality proves (i). For (ii), $x_i x_j$ is orthogonal
to all degree-$0/1$ characters and hence cannot be affine; after a sign, a
single affine unit is a linear threshold function, while two-bit parity is
not linearly separable~\citep{minsky1969perceptrons}. Full proof in
Appendix~A.
\end{quote}

The proposition concerns the \emph{all-pairs} degree-2 map and a single
layer before its output nonlinearity. The deployed \bikan{} instead uses the
channel pairs induced by a short circular-offset set $\mathcal R$, and therefore
spans a structured circulant slice rather than all $\binom{n}{2}$ characters.
Deeper binary networks remain able to compose parity-like functions across
layers. The distinction is visible synthetically: for
pairs covered by $\mathcal R=\{1,3\}$, a parity-feature linear model reaches
$100.00\pm0.00\%$ with 193 parameters, while the largest tested two-hidden-layer
sign MLP (1.12M parameters) reaches $86.95\pm1.94\%$; on unseen offsets the
same parity map falls to $50.23\pm0.33\%$. Thus our claim is efficiency, not
impossibility: explicitly exposing selected degree-2 characters provides
interactions absent from an affine binary layer, while deeper binary networks
may synthesize them indirectly.

\section{\bikan{}}
\label{sec:method}

These observations define the design target for \bikan{}: restore useful
degree-2 coordinates explicitly without materializing the full quadratic basis
or sacrificing binary efficiency. Accordingly, \bikan{} augments each layer
with a structured subset of Walsh-parity features while retaining the teacher's
hidden channel widths. Figure~\ref{fig:method} summarizes the resulting
architecture and its training-to-deployment path.

\begin{figure*}[!ht]
    \centering
    \includegraphics[width=\textwidth]{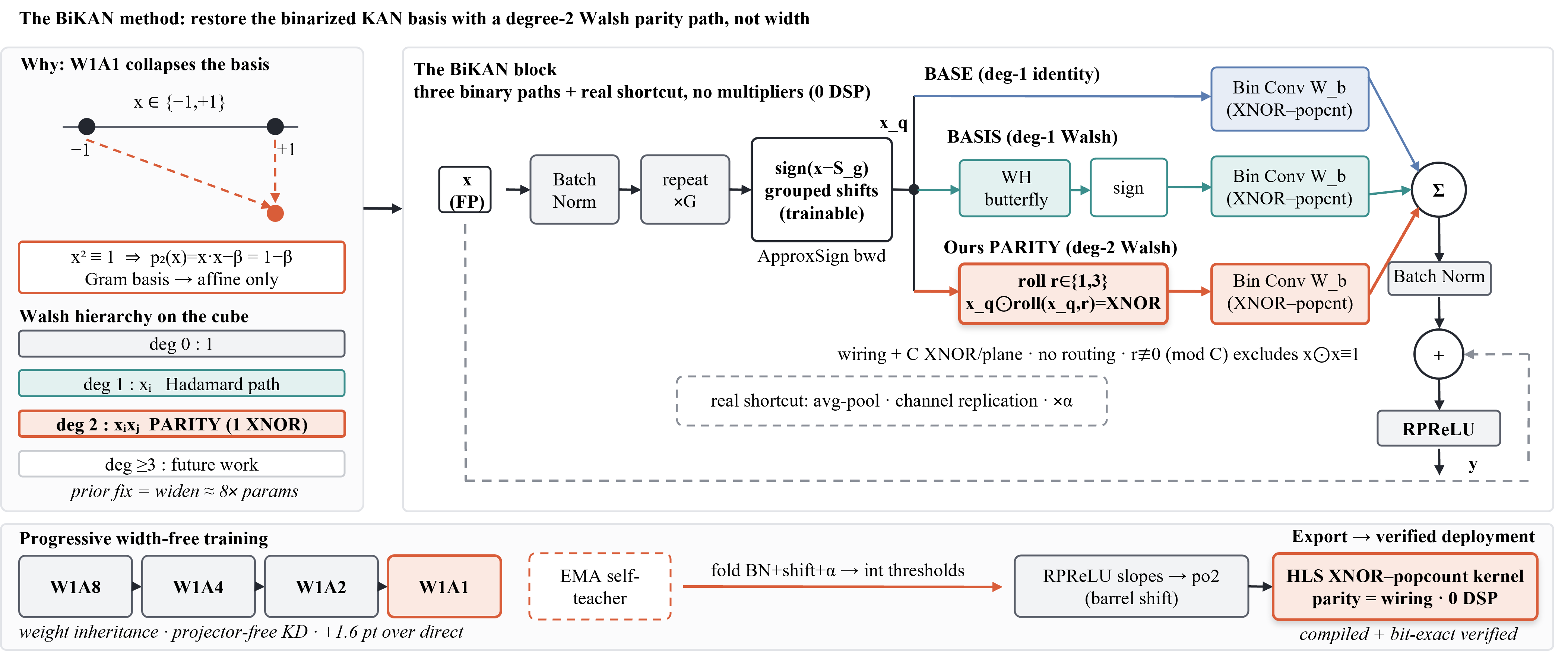}
    \caption{\textbf{The \bikan{} method.}
    W1A1 collapses an elementwise polynomial basis to constant and first-order
    responses. \bikan{} adds a structured degree-2 parity path to the binary
    base and fixed Walsh--Hadamard paths. Each circular roll forms
    $q_cq_{(c-r)\bmod C'}$, one XNOR per feature at W1A1. The three binary
    projections are accumulated before normalization, a real shortcut, and
    RPReLU. Training-only continuation, distillation, and surrogate-gradient
    components disappear at deployment.}
    \label{fig:method}
\end{figure*}

\subsection{The \bikan{} Block}
\label{sec:bikan_block}

For $\mathbf{x}\in\mathbb{R}^{C\times H\times W}$, we normalize and replicate
it into $G$ groups,
\begin{equation}
    \mathbf{x}_{\mathrm{rep}}
    =\operatorname{Rep}_{G}
    \!\left(\operatorname{BN}_{\mathrm{pre}}(\mathbf{x})\right),
    \qquad C'=GC,
    \label{eq:replicate}
\end{equation}
and quantize with trainable grouped shifts,
\begin{equation}
    \mathbf{q}
    =\operatorname{sign}(\mathbf{x}_{\mathrm{rep}}-\mathbf{S})
    \in\{-1,+1\}^{C'\times H\times W},
    \label{eq:binary_activation}
\end{equation}
using $\operatorname{sign}(0)=+1$. The shifts initialize to staggered
thresholds and are optimized jointly with the network.

The block evaluates three responses. The \emph{base path} applies a binary
convolution directly,
\begin{equation}
    \mathbf{y}_{\mathrm{base}}
    =\mathcal{F}(\mathbf{q};\widehat{\mathbf{W}}_{\mathrm{base}}),
\end{equation}
where $\mathcal{F}$ is convolution (or a dense map). The \emph{basis path}
applies a fixed block-diagonal Walsh--Hadamard transform, requantizes, and
uses a second binary projection,
\begin{align}
    \mathbf{q}_{H} &= Q_b(\mathcal{H}\mathbf{q}),\\
    \mathbf{y}_{H} &=
    \mathcal{F}(\mathbf{q}_{H};\widehat{\mathbf{W}}_{H}).
    \label{eq:hadamard_path}
\end{align}
This gives a thresholded linear remapping but no explicit pairwise feature.

\paragraph{Degree-2 parity path.}
For each circular offset $r\in\mathcal{R}$, we form
\begin{equation}
    \boldsymbol{\pi}_{r}(\mathbf{q})_c
    =\mathbf{q}_c\odot\mathbf{q}_{(c-r)\bmod C'},
    \qquad c=0,\ldots,C'-1,
    \label{eq:parity_feature}
\end{equation}
with $\odot$ applied spatially. At W1A1 this product is XNOR; the roll is
fixed wiring, so a plane needs $C'$ Boolean pair operations per spatial
location and no learned routing table. We exclude $r\equiv0\pmod{C'}$ because
$q_c^2\equiv1$, and use $\mathcal{R}=\{1,3\}$ by default. Concatenated planes
are processed by a learned binary projection,
\begin{equation}
    \mathbf{y}_{\mathrm{par}}
    =\mathcal{F}\!\left(
        \operatorname{Concat}_{r\in\mathcal{R}}
        \boldsymbol{\pi}_r(\mathbf{q});
        \widehat{\mathbf{W}}_{\mathrm{par}}
    \right).
    \label{eq:parity_path}
\end{equation}
Parity generation is parameter-free, but its projection is learned. Thus
\bikan{} is not a zero-parameter augmentation: it spends binary projection
weights on explicit higher-order coordinates while retaining the teacher's
hidden channel widths, rather than attempting to recover the same interactions
through a large width multiplier.

The paths are summed and normalized,
\begin{equation}
    \mathbf{z}
    =\operatorname{BN}_{\mathrm{post}}\!\left(
        \mathbf{y}_{\mathrm{base}}+\mathbf{y}_{H}+\mathbf{y}_{\mathrm{par}}
    \right).
    \label{eq:path_sum}
\end{equation}
A Bi-Real shortcut~\citep{liu2018bireal} carries the real-valued block input
around the binary transform; spatial reduction uses average pooling and
channel expansion uses replication. With per-channel scale
$\boldsymbol{\eta}$,
\begin{equation}
    \mathbf{z}_{r}
    =\mathbf{z}+\boldsymbol{\eta}\odot\mathcal{A}(\mathbf{x}).
    \label{eq:shortcut}
\end{equation}
Each block then applies RPReLU~\citep{liu2020reactnet},
\begin{equation}
    \operatorname{RPReLU}(z)=\operatorname{PReLU}(z-\gamma)+\zeta,
    \label{eq:rprelu}
\end{equation}
with per-channel shifts and slope; the classifier head omits this final
reshaping.

\subsection{Binarization and Optimization}
\label{sec:binarization}

\paragraph{Weights.}
We use Libra-style binarization with EDE~\citep{qin2020irnet}. Per output
channel,

\begin{equation}
\small
\widetilde{\mathbf{W}}
=\frac{\mathbf{W}-\mu_W}{\sigma_W+\epsilon},\;
\widehat{\mathbf{W}}
=\alpha_W\operatorname{sign}(\widetilde{\mathbf{W}}),\;
\alpha_W=\mathbb{E}[|\widetilde{\mathbf{W}}|].
\label{eq:weight_bin}
\end{equation}
Standardization leaves the forward sign pattern of mean-centering unchanged
but normalizes the latent scale used by $\alpha_W$ and the surrogate. EDE uses
\begin{equation}
    \frac{\partial \widehat W}{\partial W}
    \approx t\left[1-\tanh^2(t\widetilde W)\right],
    \label{eq:ede}
\end{equation}
with $t$ annealed toward a sharper sign approximation. Deployed weights remain
$\{-\alpha_W,+\alpha_W\}$.

\paragraph{Activations and grouped thresholds.}
The W1A1 backward pass uses the Bi-Real ApproxSign
surrogate~\citep{liu2018bireal},
\begin{equation}
    \frac{\partial q}{\partial u}
    \approx
    \begin{cases}
        2-2|u|, & |u|<1,\\
        0,      & \text{otherwise},
    \end{cases}
    \qquad u=x-S.
    \label{eq:approxsign}
\end{equation}
Gradients also update the grouped shifts. To discourage redundant thresholds,
we use the training-only hinge
\begin{equation}
    \mathcal{L}_{\mathrm{div}}
    =\sum_{\ell}\operatorname*{mean}_{c}
    \max\!\left(0,m-\operatorname{Std}_{g}S^{(\ell)}_{g,c}\right).
    \label{eq:shift_diversity}
\end{equation}

\subsection{Teacher-Guided Precision Descent}
\label{sec:training}

Because student and FP32 teacher share hidden widths, intermediate attention
maps align without learned projectors. We combine cross-entropy and
$T$-scaled logit distillation,
\begin{equation}
    \mathcal{L}_{\mathrm{KD}}
    =(1-\alpha)\mathcal{L}_{\mathrm{CE}}
    +\alpha T^2D_{\mathrm{KL}}\!\left(p_T^{(T)}\Vert p_S^{(T)}\right),
\end{equation}
and optimize
\begin{equation}
    \mathcal{L}
    =\mathcal{L}_{\mathrm{KD}}
    +\lambda_{\mathrm{AT}}\mathcal{L}_{\mathrm{AT}}
    +\lambda_{\mathrm{div}}\mathcal{L}_{\mathrm{div}}
    +\lambda_{\mathrm{EMA}}\mathcal{L}_{\mathrm{EMA}},
    \label{eq:training_objective}
\end{equation}
with scale-invariant attention transfer at the three intermediate feature taps
and the EMA self-teacher active only at W1A1. Because the feature dimensions
already match, attention transfer introduces no learned projectors. Latent
binary weights use zero weight decay and a separate learning-rate multiplier;
other parameters use the standard optimizer settings. BatchNorm statistics are
recalibrated from training data after optimization. Weights remain 1-bit while activation precision descends through W1A8$\rightarrow$W1A4$\rightarrow$W1A2$\rightarrow$W1A1, with latent-state inheritance. The multi-bit stages provide optimization continuation rather than additional capacity or deployment targets. Because the ladder also adds optimization steps, we compare it with direct W1A1 training matched for total compute instead of attributing the short-budget gap to the schedule itself.

\subsection{Hardware Mapping}
\label{sec:hardware_mapping}

For W1A1 operands, binary inner products reduce to XNOR--popcount. If $M$ of
$N$ bit pairs agree,
\begin{equation}
    \sum_{i=1}^{N}q_i b_i
    =2M-N,\qquad
    M=\operatorname{popcount}\!\left(
    \operatorname{XNOR}(\mathbf q,\mathbf b)\right).
    \label{eq:xnor_popcount}
\end{equation}

Thus all three learned projections share the same binary core; Hadamard uses
fixed add/subtract logic, while parity uses fixed channel rolls plus XNOR.
This core is multiplier-free.

End-to-end inference can nevertheless retain fixed-point scale arithmetic:
binary-weight scales, post-normalization, the real shortcut, and the classifier
may require materialized scaling rather than immediate threshold folding. We
therefore distinguish \emph{multiplier-free binary compute} from
\emph{zero-DSP end-to-end inference} and report the latter from synthesis.
Power-of-two scales replace eligible multiplies by shifts. For the dense GRAM
hardware variant, Po2-aware fine-tuning exposes rounded power-of-two weight
scales in the forward binarizer with a straight-through rounding gradient so
latent weights can adapt to the exported scales. In the convolutional backbone,
normalization and shortcut scales form additional scale families outside the
weight binarizer, so adapting $\alpha_W$ alone does not imply zero-DSP
end-to-end inference. EDE, ApproxSign, distillation, EMA, and diversity
regularization are training-only and leave no inference-time hardware
footprint.

\section{Experiments}
\label{sec:experiments}

\paragraph{Evaluation protocol.}
We evaluate \bikan{} using a pre-registered, multi-seed protocol that controls
for seed variance, checkpoint selection, width, and optimization budget.
Compared models share the same seed and FP32 teacher checkpoint. We select
checkpoints on a held-out validation split and report BN-recalibrated test
accuracy. Primary CIFAR-10 comparisons use five paired seeds; secondary studies
and CIFAR-100 use three. Results are reported as mean$\pm$standard deviation,
with two-sided paired $t$-tests for seed-matched comparisons. Primary runs are
colocated on the same accelerator to avoid cross-GPU effects. Earlier
single-seed tabular/MLP and cross-family studies use test-selected checkpoints
and are reported only as exploratory evidence. Full protocols, hypotheses,
per-seed results, and preliminary experiments appear in Appendices~B--J.

\subsection{W1A1 Accuracy}
\label{sec:main_results}
Table~\ref{tab:main} asks whether the representation restored by \bikan{}
remains effective at W1A1. On CIFAR-10, \bikan{} exceeds its paired FP32
teachers by $1.26$ points ($84.38\pm0.14\%$ vs.\
$83.11\pm0.34\%$, $p=6.6\times10^{-4}$). CIFAR-100 shows a
$2.57$-point improvement ($p=0.047$), although the variance and $n=3$
warrant caution. MNIST is saturated, with no significant difference
($p=0.26$). On Tiny-ImageNet, \bikan{} improves by $4.20$ points over two
seeds; because these runs use a final-checkpoint fallback without validation
selection, we treat them as supporting breadth rather than a significance
claim.

\begin{table}[!ht]
\centering
\caption{\textbf{Primary W1A1 classification results.}
BN-recalibrated test accuracy (\%) from paired, validation-selected
checkpoints.}
\label{tab:main}
\scriptsize
\setlength{\tabcolsep}{3.2pt}
\begin{tabular}{@{}lcccc@{}}
\toprule
 & MNIST & CIFAR-10 & CIFAR-100 & T-IN$^\dagger$ \\
\midrule
FP32 teacher
& $99.53{\pm}0.01$
& $83.11{\pm}0.34$
& $53.24{\pm}2.83$
& $32.95{\pm}1.80$ \\
\bikan{} W1A1
& $\mathbf{99.48{\pm}0.05}$
& $\mathbf{84.38{\pm}0.14}$
& $\mathbf{55.81{\pm}1.86}$
& $\mathbf{37.15{\pm}0.68}$ \\
$n$
& 3 & 5 & 3 & 2 \\
\midrule
$\Delta$ vs.\ teacher
& $-0.05$
& $\mathbf{+1.26}$
& $\mathbf{+2.57}$
& $+4.20$ \\
\bottomrule
\end{tabular}

\vspace{1mm}
{\scriptsize $^\dagger$Tiny-ImageNet uses final-checkpoint rather than
validation-based selection and is supporting evidence only.}
\end{table}

These results show that W1A1 preserves the accuracy of the underlying KAN.
We next test the central mechanism: whether explicit degree-2 Walsh
coordinates efficiently restore capacity lost through basis collapse.

\subsection{Does Restoring Parity Matter?}
\label{sec:mechanism_results}

\paragraph{A reproducible parity effect.}
Table~\ref{tab:ablation} isolates parity under the same direct-to-W1A1
training budget. Removing it reduces CIFAR-10 accuracy from
$83.00\pm0.31\%$ to $81.77\pm0.38\%$; all five paired seeds favor
\bikan{}, yielding a $+1.23\pm0.43$-point gain
($t(4)=6.49$, $p=0.003$).

\begin{table}[!ht]
\centering
\caption{\textbf{Mechanism ablation on CIFAR-10.}
All arms use the same direct-to-W1A1 budget.}
\label{tab:ablation}
\scriptsize
\setlength{\tabcolsep}{3.5pt}
\begin{tabular}{@{}lccc@{}}
\toprule
Configuration & Accuracy (\%) & $n$ & Params. \\
\midrule
\textbf{Full \bikan{}}
& $\mathbf{83.00{\pm}0.31}$ & 5 & 11.94M \\
w/o parity
& $81.77{\pm}0.38$ & 5 & 5.97M \\
w/o shortcut
& $83.08{\pm}0.28$ & 3 & 11.94M \\
w/o RPReLU
& $78.93{\pm}0.85$ & 3 & 11.94M \\
frozen STE
& $83.05{\pm}0.29$ & 3 & 11.94M \\
bare binary KAN
& $78.19{\pm}0.33$ & 5 & 5.97M \\
\bottomrule
\end{tabular}
\end{table}

Removing RPReLU costs about $4.1$ points, confirming its importance for
W1A1 optimization, although it is inherited from prior binary-network work.
Removing the shortcut or freezing the estimator has no measurable effect at
this budget. The contribution specific to \bikan{} is therefore the additional
$1.23$-point gain from explicit parity beyond standard binary-training
machinery.

\paragraph{Dose response.}
The basis-restoration hypothesis predicts that exposing more degree-2
coordinates should increase useful capacity. Figure~\ref{fig:dose} tests this
by varying only $|\mathcal{R}|$.

\begin{figure}[!ht]
\centering
\includegraphics[width=0.75\columnwidth]{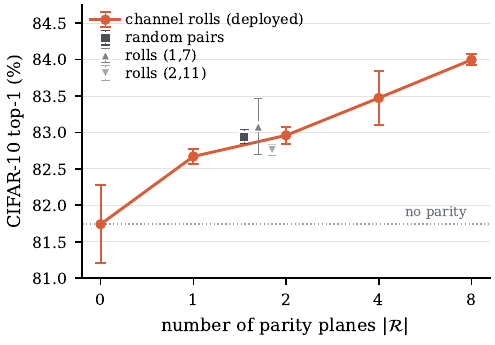}
\caption{\textbf{Parity dose response and pairing controls on CIFAR-10.}
Accuracy rises with parity-plane count
($n=5$ for $|\mathcal{R}|=2$, otherwise $n=3$).
Random and alternative circular pairings perform comparably to the deployed
hardware-friendly rolls. Exact values are in Appendix~D.}
\label{fig:dose}
\end{figure}

Accuracy rises monotonically from $81.74\%$ to $84.00\%$ as
$|\mathcal{R}|$ increases from $0$ to $8$, without saturating. Random pairing
and alternative offset sets remain close to the deployed two-plane result,
indicating no systematic penalty from the fixed circulant wiring.

\paragraph{Parity matters most when capacity is scarce.}
If parity restores per-neuron capacity, its benefit should grow as width
shrinks. Table~\ref{tab:capacity} confirms this prediction: the parity gap
increases from $+1.23$ points at full width to $+1.66$ at half width and
$+2.87$ at quarter width, with all comparisons significant.

\begin{table}[!ht]
\centering
\caption{\textbf{Capacity and width controls on CIFAR-10.}
The upper block measures parity under reduced width; the lower block compares
parity with conventional widening.}
\label{tab:capacity}
\scriptsize
\setlength{\tabcolsep}{3.0pt}
\begin{tabular}{@{}lccc@{}}
\toprule
Comparison & Difference & $n$ & $p$ \\
\midrule
Parity gap, $0.25\times$ width
& $+2.87{\pm}0.67$ & 3 & 0.018 \\
Parity gap, $0.50\times$ width
& $+1.66{\pm}0.40$ & 3 & 0.019 \\
Parity gap, $1.00\times$ width
& $+1.23{\pm}0.43$ & 5 & 0.003 \\
\midrule
Parity vs.\ width, equal params.
& $\mathbf{+3.09{\pm}0.31}$ & 5 & $<10^{-4}$ \\
Full $1\times$ vs.\ bare $4\times$
& $+0.33{\pm}0.49$ & 5 & 0.210 \\
\bottomrule
\end{tabular}
\end{table}

The equal-parameter control separates structured restoration from generic
capacity. With comparable budgets, the two-plane parity model reaches
$82.96\pm0.12\%$ at 11.94M parameters, versus $79.87\pm0.27\%$ for a
widened bare model at 11.77M, a $+3.09$-point advantage
($p<10^{-4}$). Thus explicit degree-2 coordinates are more effective than
ordinary width at this budget.

\paragraph{Width still helps.}
Parity improves efficiency rather than eliminating the value of width.
Under direct training, teacher-width \bikan{} and a 94.67M-parameter bare
$4\times$ model are statistically indistinguishable
($83.00\pm0.31\%$ vs.\ $82.68\pm0.21\%$, $p=0.21$). With the full
training recipe, the widened model reaches $84.76\pm0.02\%$ over two seeds,
about $0.4$ points above \bikan{}. The claim is therefore an improved
accuracy--capacity trade-off, not universal dominance over widening.

\subsection{From Proposition to Measurement}
\label{sec:synthetic_results}

A controlled Boolean task tests whether parity helps for the mechanism predicted
by the analysis. Inputs lie in $\{-1,+1\}^{64}$, and each target depends on
seven degree-2 interactions. The \emph{covered} task uses pair distances within
the orbit of the deployed $\mathcal{R}=\{1,3\}$ rolls; the
\emph{uncovered} task uses distances outside it.

\begin{table}[!ht]
\centering
\caption{\textbf{Synthetic test of the Walsh interpretation.}
Mean$\pm$std over three seeds. The parity map solves only interactions exposed
by its roll orbit.}
\label{tab:synthetic_main}
\scriptsize
\setlength{\tabcolsep}{2.5pt}
\begin{tabular}{@{}lrrr@{}}
\toprule
Model & Covered & Uncovered & Params. \\
\midrule
FP linear
& $49.9{\pm}1.5$ & $50.0{\pm}0.3$ & 65 \\
Parity map, $\mathcal{R}=\{1,3\}$
& $\mathbf{100.0{\pm}0.0}$ & $50.2{\pm}0.3$ & \textbf{193} \\
1-layer sign, width 1024
& $50.0{\pm}0.2$ & $49.9{\pm}1.5$ & 67.6K \\
2-layer sign, width 1024
& $87.0{\pm}1.9$ & $86.4{\pm}1.1$ & 1.12M \\
\bottomrule
\end{tabular}
\end{table}

The 193-parameter parity map reaches $100\%$ on covered interactions, while
the linear and one-hidden-layer sign models remain at chance. A deeper sign
network can synthesize the interactions indirectly, but the 1.12M-parameter
model reaches only $87.0\%$. Conversely, the parity map remains at chance
($50.2\%$) on uncovered pairs. Thus \bikan{} provides direct access to
selected degree-2 Walsh coordinates without spanning all pairwise characters
or precluding deeper composition. All pre-registered criteria are satisfied;
Appendix~D reports the full width sweep.

\subsection{Separating Architecture from Optimization}
\label{sec:optimization_results}

\paragraph{Controlling for optimization budget.}
Progressively reducing activation precision from
W1A8$\rightarrow$W1A4$\rightarrow$W1A2$\rightarrow$W1A1 produces a strong
binary endpoint: on CIFAR-10, the final model reaches
$84.38\pm0.14\%$, compared with $82.51\pm0.25\%$ for direct W1A1 training
given the same 40-epoch final-stage budget ($+1.87$ points,
$p<10^{-4}$). This comparison, however, gives the progressive trajectory
more total optimization steps. When direct W1A1 training is instead matched
to the same 160-epoch total compute, it reaches $84.43\pm0.29\%$; over the
matched seeds the difference is only $-0.08\pm0.41$ points ($p=0.77$).
We therefore do not attribute an independent accuracy gain to progressive
descent. Its practical role is as a continuation strategy that yields usable
intermediate-precision checkpoints---W1A8, W1A4, and W1A2---along the same
training trajectory before reaching the final W1A1 deployment point. The
complete precision trajectory and recipe controls are reported in
Appendix~F.



\subsection{Breadth Beyond Convolutional Benchmarks}
\label{sec:tabular_results}

The primary studies above isolate the mechanism under a modern paired protocol.
We next ask a narrower breadth question: when width and training budget are
held fixed, does a binary KAN representation retain an advantage over an
ordinary binary MLP outside the convolutional KAGN setting? Table~\ref{tab:tabular}
reports the available tabular and dense-MLP suite. These runs use one seed and
the earlier test-selected protocol, so they are exploratory and are not used
for significance claims.

\begin{table}[!ht]
\centering
\caption{\textbf{Exploratory tabular and dense-MLP results at W1A1.}
Classification entries are accuracy (\%, higher is better); Traffic reports
RMSE (lower is better). All rows use one seed under the earlier protocol.
Detailed architectures, parameter counts, and recalibration behavior are in
Appendix~I.}
\label{tab:tabular}
\scriptsize
\setlength{\tabcolsep}{2.4pt}
\begin{tabular}{@{}lrrrr@{}}
\toprule
Dataset & FP32 KAN & \bikan{} & Wide $4\times$ & Bin.\ MLP \\
\midrule
Wine
& 66.67 & 61.11 & 66.67 & 47.22 \\
Dry Bean
& 91.30 & 83.14 & 90.19 & 52.19 \\
JSC OpenML
& 76.77 & 71.67 & 71.76 & 53.43 \\
Tiny-ImageNet MLP
& 10.85 & 10.63 & 12.12 & 7.63 \\
Traffic, RMSE $\downarrow$
& 0.1001 & 0.1082 & 0.1063 & 0.1590 \\
\bottomrule
\end{tabular}
\end{table}

Across all five tasks, \bikan{} is better than the matched binary MLP:
the gains are $+13.89$, $+30.95$, $+18.24$, and $+3.00$ accuracy points
on the four classification datasets, while Traffic RMSE improves from
$0.1590$ to $0.1082$. This pattern suggests that the KAN representation still
earns its additional structure after binarization; the result is not confined
to a convolutional backbone. The comparison to widening is deliberately more
qualified. Teacher-width \bikan{} essentially ties the wide model on JSC
($71.67$ vs.\ $71.76$), but widening is better on Wine, Dry Bean,
Tiny-ImageNet MLP, and Traffic. Thus the suite supports a basis-versus-binary-MLP
advantage, not a universal width-free claim.

A small cross-family check points in the same direction: teacher-width W1A1
students reach 96.85\% versus 97.66\% for EfficientKAN and 96.79\% versus
97.63\% for PyKAN, each with 0.82M parameters. These are also single-seed
legacy results, and the attempted FastKAN run produced no valid checkpoint;
we therefore retain their full context in Appendix~J rather than treating them
as primary evidence.

\subsection{FPGA Realization}
\label{sec:hardware_results}

Finally, we ask whether restoring degree-2 coordinates gives back accuracy by
reintroducing expensive arithmetic. It does not: parity generation is fixed
channel wiring plus XNOR, and the learned projection uses the same
XNOR--popcount primitive as the remaining binary paths. Table~\ref{tab:hw}
reports post-place-and-route resources on a Xilinx Zynq-7020 for two \bikan{}
realizations that expose complementary hardware regimes: a dense GRAM model,
where almost all large inner products can be made binary, and the convolutional
KAGN backbone used in the main accuracy experiments.

\begin{table}[!ht]
\centering
\caption{\textbf{Post-place-and-route FPGA results on Zynq-7020.}
Dense GRAM latency is cycle-derived from C/RTL cosimulation; convolutional
latency is the common HLS \texttt{csynth-min} compute-core estimate and is not
an end-to-end board measurement. BRAM denotes 18-Kb blocks.}
\label{tab:hw}
\scriptsize
\setlength{\tabcolsep}{2pt}
\renewcommand{\arraystretch}{1.05}
\resizebox{\columnwidth}{!}{%
\begin{tabular}{@{}llrrrrrrc@{}}
\toprule
Backbone & Design & Acc.\,(\%) & DSP & LUT & FF & BRAM & F$_{\max}$ & Latency \\
\midrule
GRAM & FP32
& 97.64 & 104 & 14,417 & 13,447 & 30 & 100.7 & 59.6 ms \\
GRAM & W1A1
& 97.64 & 6 & 9,365 & 4,791 & 272 & 98.9 & \textbf{0.705 ms} \\
GRAM & W1A1 + Po2-QAT
& \textbf{97.61} & \textbf{0} & 9,700 & 5,170 & 272 & 98.9 & 0.705 ms \\
\midrule
KAGN conv & FP32
& 99.55 & 164 & 25,155 & 23,929 & 224 & 103.0 & 401 ms \\
KAGN conv & \bikan{} W1A1
& \textbf{99.49} & 72 & \textbf{9,144} & \textbf{11,482} & 215 & 110.0 & \textbf{54.8 ms} \\
KAGN conv & W1A1 + naive Po2
& 98.61 & \textbf{12} & 11,555 & 13,746 & \textbf{200} & \textbf{115.8} & 54.8 ms \\
\bottomrule
\end{tabular}%
}
\end{table}

The dense result shows the upper end of the binary-hardware opportunity. W1A1
preserves the 97.64\% FP32 accuracy while reducing DSP usage from 104 to 6
($94\%$), LUTs from 14.4k to 9.4k, and RTL-cosim latency from 59.6 to
0.705\,ms ($84\times$). Its larger BRAM count reflects a different storage
mapping---more of the binary weights are kept on chip---rather than a uniform
increase in hardware cost. Raw binary-weight storage is also favorable despite
the parity-expanded parameter count: the dense FP32 model is approximately
2.0\,MB, whereas 3.26M one-bit weights require about 0.41\,MB before scales
and metadata.

The convolutional implementation retains more non-binary operations around its
XNOR cores, but the reduction remains substantial: DSPs fall from 164 to 72,
LUTs from 25.2k to 9.1k, and FFs from 23.9k to 11.5k while accuracy changes by
only $-0.06$ points. Under the common HLS scheduling model, compute-core
latency falls from 401 to 54.8\,ms ($7.3\times$). The W1A1 RTL cosimulation
that completes is DDR-bound (450.9\,ms); we therefore use it only as
verification evidence and do not mix it with the csynth-based FP32 comparison.

\paragraph{Can the remaining multipliers be removed?}
The XNOR--popcount dot product itself uses zero DSPs; residual DSPs arise from
scales, normalization, shortcut, head, and indexing arithmetic. This makes
power-of-two scale mapping a direct stress test of the hardware story. Naively
rounding trained GRAM scales collapses accuracy to 55.46\%, showing that
multiplier removal is not a free post-processing step. When the same rounding
is placed inside a short fine-tuning stage, the model adapts and recovers
97.61\%, while post-route DSP usage falls from 6 to \emph{exactly zero}. The
cost is only 335 additional LUTs, or roughly 56 LUTs per eliminated DSP. For
the convolutional design, naive Po2 mapping already reduces DSPs from 72 to 12
and raises F$_{\max}$ to 115.8\,MHz, but costs 0.88 points because folded
normalization and shortcut scales are not yet Po2-constrained during training.
We report this as an explicit hardware--accuracy operating point rather than a
lossless result.

All exported variants are checked against the software integer reference on
all 10,000 MNIST test images, and C/RTL cosimulation verifies equivalence
between the generated C kernel and RTL. These are implementation-correctness
checks; we make no claim of measured board-level energy or end-to-end latency.
Full EfficientKAN controls, residual-DSP accounting, and verification details
are given in Appendix~K.

\paragraph{What the experiments establish.}
The evidence converges on a specific mechanism. Removing parity costs
$1.23$ points; adding parity planes improves accuracy monotonically; the
benefit grows from $1.23$ to $2.87$ points as ordinary width is starved; and,
at an equal $\sim 11.9$M-parameter budget, parity outperforms width by
$3.09$ points. At the same time, the controls identify the limits of the
claim: widening can still improve accuracy, the progressive schedule has no
advantage after compute is matched, and the deployed circular offsets expose
only a structured subset of all degree-2 Walsh characters. The exploratory non-convolutional suite adds a complementary
breadth result: the binary KAN exceeds a matched binary MLP on every tested
task, while its comparison to a widened KAN is mixed. Together, these results
support the intended interpretation of \bikan{}: it repairs a specific
representational loss of W1A1 KANs rather than compensating for that loss
primarily through brute-force width or training heuristics, without claiming
that parity eliminates the value of width across all domains.

\section{Discussion and Limitations}
\label{sec:limitations}

The strongest mechanism evidence comes from the paired CIFAR-10 experiments,
while CIFAR-100, Tiny-ImageNet, and the single-seed non-convolutional studies
provide supporting breadth. Our results establish an efficiency advantage over
widening rather than universal dominance, since sufficiently wide binary
networks can recover additional accuracy at substantially greater parameter
cost. The current implementation also exposes a structured circulant subset of
degree-2 Walsh characters rather than the complete quadratic basis. Finally,
the FPGA evaluation is limited to one device and benchmark; resource counts are
post-route, whereas convolutional latency is an HLS compute-core estimate rather
than a board-level measurement. Within this scope, the experiments consistently
show that restoring selected parity coordinates provides a reproducible and
parameter-efficient repair for the representational loss induced by W1A1.

\section{Conclusion}
\label{sec:conclusion}

Binarizing a polynomial KAN alters more than coefficient precision: on the
Boolean domain, its elementwise polynomial basis collapses to constant and
first-order responses. \bikan{} addresses this failure by exposing selected
degree-2 Walsh parities through fixed circular wiring and XNOR, followed by
learned binary projections. The results support this basis-restoration account:
parity survives paired multi-seed ablation, improves with plane count, becomes
more valuable as width decreases, and outperforms conventional widening at a
matched parameter budget. These gains remain compatible with XNOR--popcount
inference and substantial post-route FPGA resource reductions. \bikan{} restores a structured subset of pairwise interactions rather than the complete Boolean basis, and additional width remains complementary. Future work can explore learned parity selection, higher-order Walsh characters, and hardware-aware training that extends multiplier-free execution across the full convolutional pipeline.

\bibliography{aaai2027}


\clearpage

\appendix

This appendix provides the theoretical, statistical, and implementation detail
behind the expanded main-paper evaluation. Appendix~A proves Proposition~1.
Appendices~B and C document the pre-registered evaluation protocol, the complete
primary-dataset summary, per-seed parity evidence, and all paired statistics.
Appendix~D expands the main paper's parity dose--response and synthetic
Walsh tests, including the controls that delimit the circulant construction.
Appendix~E gives the full capacity-starvation and width frontier, while
Appendix~F separates training-recipe effects from optimization budget and
reports the complete precision trajectory. Appendix~G preserves preliminary
single-seed experiments under their original selection protocol so they are
never mixed with the primary results. Appendices~H--J contain secondary MNIST,
tabular, and cross-family studies. Appendix~K expands the main-paper FPGA table
with cross-backbone controls, derived resource reductions, exact latency
provenance, residual-DSP accounting, power-of-two experiments, and verification
cycles. Appendix~L collects the scope and limitations of the evidence.

\section{Proof of Proposition 1}

We work on the Boolean cube $\{-1,+1\}^{n}$, $n\geq2$, under the uniform
measure. For $S\subseteq[n]$, define the Walsh character
\[
    \chi_S(\mathbf{x})=\prod_{i\in S}x_i.
\]
The characters form an orthonormal basis for the real-valued functions on the
Boolean cube:
\[
    \langle\chi_S,\chi_T\rangle
    =
    \mathbb{E}_{\mathbf{x}}
    [\chi_S(\mathbf{x})\chi_T(\mathbf{x})]
    =
    \mathbf{1}[S=T].
\]
Indeed,
$\chi_S\chi_T=\chi_{S\triangle T}$ and
$\mathbb{E}[\chi_U]=0$ for every nonempty $U$. Hence every
$f:\{-1,+1\}^n\rightarrow\mathbb{R}$ admits the unique expansion
\[
    f(\mathbf{x})
    =
    \sum_{S\subseteq[n]}
    \widehat f(S)\chi_S(\mathbf{x}),
\]
and its Walsh degree is the largest $|S|$ for which
$\widehat f(S)\neq0$.

\paragraph{Part (i).}
The feature set
\[
    \{1\}
    \cup
    \{x_i\}_{i=1}^{n}
    \cup
    \{x_i x_j\}_{1\leq i<j\leq n}
\]
is exactly
$\{\chi_S:|S|\leq2\}$:
$1=\chi_{\emptyset}$,
$x_i=\chi_{\{i\}}$, and
$x_i x_j=\chi_{\{i,j\}}$.
Orthonormality makes these
$1+n+\binom{n}{2}$ functions linearly independent. Their span is therefore
precisely the subspace of pseudo-Boolean functions of Walsh degree at most
two. Consequently, a linear readout over the complete feature map
$\Phi_2$ represents every such function exactly by choosing each coefficient
equal to the corresponding Fourier coefficient. \hfill$\square$

\paragraph{Part (ii), before the sign.}
Fix $i\neq j$. Every affine function has the expansion
\[
    g(\mathbf{x})
    =
    b\chi_{\emptyset}
    +
    \sum_k w_k\chi_{\{k\}},
\]
and hence has Fourier support only in degrees zero and one. The target
$\chi_{\{i,j\}}$ has all Fourier mass at degree two. Uniqueness of the Walsh
expansion therefore implies
$g\neq\chi_{\{i,j\}}$ for every $\mathbf{w},b$.
More strongly, orthogonality gives
\[
    \|g-\chi_{\{i,j\}}\|_2^2
    =
    \|g\|_2^2+1
    \geq1.
\]
Increasing the width of an affine layer only produces more coordinates with
the same degree-$\leq1$ support and does not change this statement.
\hfill$\square$

\paragraph{Part (ii), after the sign.}
A single unit
$\operatorname{sign}(\mathbf{w}^{\top}\mathbf{x}+b)$
is a linear threshold function. Restricting to $(x_i,x_j)$ while fixing all
other coordinates, computing $x_i x_j$ requires output $+1$ on
$\{(+,+),(-,-)\}$ and $-1$ on $\{(+,-),(-,+)\}$, i.e.,
the two-bit XOR pattern. These two classes are not linearly separable
\citep{minsky1969perceptrons}. \hfill$\square$

\paragraph{Scope.}
The proposition describes the complete degree-2 feature map and a single
layer at its sign boundary. The deployed \bikan{} is intentionally smaller:
it exposes only the pairwise characters generated by the circular offsets
$\mathcal{R}=\{1,3\}$. Deeper binary networks can also synthesize pairwise
interactions through composition. Appendix~D measures both qualifications
directly rather than treating them as assumptions.

\section{Evaluation Protocol and Pre-Registered Tests}

The validation campaign comprises 117 jobs using an H100 and a V100 over a
shared filesystem. Before the corresponding full runs, the analysis plan fixed
the principal hypotheses, decision thresholds, statistics, and reporting
fallbacks. The committed plan is retained as 	\texttt{PREREGISTRATION.md} in the
supplementary artifacts so that the reported thresholds can be audited against
the decisions specified before the full campaign.

\paragraph{Hypotheses.}
The planned tests were:

\begin{itemize}
    \item \textbf{H1, parity:}
    full minus no-parity is at least $0.5$ points with $p<0.05$.

    \item \textbf{H2, dose response:}
    mean accuracy is non-decreasing over
    $|\mathcal{R}|\in\{0,1,2,4,8\}$.

    \item \textbf{H3, mechanism versus width:}
    for full-$1\times$ minus bare-$4\times$, a 95\% CI lower bound above
    $-0.25$ supports ``matches,'' while a bound above zero supports
    ``exceeds.''

    \item \textbf{H3$'$, equal parameters:}
    parity minus a parameter-matched widened bare model is positive with
    $p<0.05$.

    \item \textbf{H4, progressive training:}
    progressive minus direct W1A1 at the final-stage budget is at least
    $1.0$ point with $p<0.05$; the total-compute-matched comparison is
    reported regardless of outcome.

    \item \textbf{H5, capacity starvation:}
    the paired parity gap increases as width is reduced.

    \item \textbf{H6, synthetic task:}
    covered-parity accuracy is at least 99\%, the FP linear baseline is at
    most 55\%, and uncovered-parity accuracy is at most 60\%.
\end{itemize}

The reporting rule was fixed in advance: a failed threshold is reported as
negative or inconclusive rather than replaced with a post-hoc criterion.

\paragraph{Checkpoint selection.}
Each primary dataset uses a held-out validation split with a fixed split seed.
Checkpoints are selected using validation accuracy only. The primary reported
metric is BN-recalibrated test accuracy of that validation-selected checkpoint.
The earlier preliminary experiments instead selected on the test set; those
runs are isolated in Appendix~G and are not pooled with the primary results.

\paragraph{Pairing and statistics.}
Seeds are paired across compared arms. For every dataset and seed, compared
students share the same FP32 teacher checkpoint. We use two-sided paired
$t$-tests at $\alpha=0.05$ and report mean$\pm$standard deviation with the
number of seeds. Primary CIFAR-10 comparisons use five seeds; secondary arms
and CIFAR-100 use three where stated.

\paragraph{Machine colocation.}
Every primary paired comparison was colocated on one accelerator. Hence
H100/V100 differences in TF32 behavior, cuDNN kernels, or runtime details
cannot masquerade as treatment effects. Only several secondary recipe arms
cross machines; their hardware provenance is retained with the artifacts.

\paragraph{Tiny-ImageNet exception.}
Tiny-ImageNet uses two seeds and the existing
\texttt{val\_fraction=0} fallback, with the final checkpoint rather than a
validation-selected checkpoint. It is therefore treated as supporting
evidence rather than as a primary statistical result.

\paragraph{Outcomes.}
H1, H2, H3$'$, H5, and H6 pass their stated criteria. H3 misses its
pre-registered ``matches'' boundary narrowly: the CI lower bound is $-0.28$
rather than the required $-0.25$, so the main paper calls the two models
statistically indistinguishable rather than claiming superiority. H4 passes
at equal final-stage budget but disappears when total optimization compute is
matched; the main paper reports both comparisons.

\section{Primary Results and Complete Paired Statistics}

\subsection{Primary Dataset Summary}

Table~\ref{tab:primaryapp} mirrors the expanded main-paper classification table
and makes the evidence hierarchy explicit. CIFAR-10 is the strongest dataset
for inferential claims because it uses five paired seeds. CIFAR-100 uses three
paired seeds and exhibits larger variance. Tiny-ImageNet uses two seeds and the
final-checkpoint fallback described in Appendix~B, so it is included only as
supporting breadth.

\begin{table*}[t]
\centering
\scriptsize
\setlength{\tabcolsep}{5pt}
\begin{tabular}{@{}lrrrrrr@{}}
\toprule
Dataset & FP32 teacher & \bikan{} W1A1 & $\Delta$ & $n$ & Paired $p$ & Student params. \\
\midrule
MNIST
& $99.53{\pm}0.01$
& $99.48{\pm}0.05$
& $-0.05$
& 3
& 0.260
& 11.90M \\

CIFAR-10
& $83.11{\pm}0.34$
& $\mathbf{84.38{\pm}0.14}$
& $\mathbf{+1.26}$
& 5
& $6.6\times10^{-4}$
& 11.94M \\

CIFAR-100
& $53.24{\pm}2.83$
& $\mathbf{55.81{\pm}1.86}$
& $\mathbf{+2.57}$
& 3
& 0.047
& 12.68M \\

Tiny-ImageNet$^\dagger$
& $32.95{\pm}1.80$
& $37.15{\pm}0.68$
& $+4.20$
& 2
& --
& 13.50M \\
\bottomrule
\end{tabular}
\caption{\textbf{Primary W1A1 dataset summary.}
The primary metric is BN-recalibrated test accuracy of the validation-selected
checkpoint. $^\dagger$Tiny-ImageNet uses the no-validation-split fallback and
is not assigned a headline significance claim.}
\label{tab:primaryapp}
\end{table*}

The CIFAR-10 and CIFAR-100 rows should not be interpreted as showing that
precision reduction by itself improves a fixed FP32 function. The student also
changes the representation through explicit parity coordinates and is trained
with teacher supervision. The dataset-level result establishes that this W1A1
representation is competitive; the mechanism studies below establish why.

\subsection{Per-Seed Parity Evidence}

The central on/off mechanism test is reported seed by seed in
Table~\ref{tab:parityseeds}. Every paired seed favors the parity model. This is
useful context for the mean effect because the individual gain ranges from
$0.69$ to $1.68$ points; a single seed can therefore materially understate or
overstate the effect.

\begin{table}[t]
\centering
\scriptsize
\setlength{\tabcolsep}{5pt}
\begin{tabular}{@{}rrrr@{}}
\toprule
Seed & Full \bikan{} & No parity & Paired difference \\
\midrule
0 & 82.94 & 81.34 & $+1.60$ \\
1 & 83.42 & 81.74 & $+1.68$ \\
2 & 82.68 & 81.99 & $+0.69$ \\
3 & 82.77 & 81.51 & $+1.26$ \\
4 & 83.21 & 82.28 & $+0.93$ \\
\midrule
Mean & 83.00 & 81.77 & $\mathbf{+1.23}$ \\
Std. & 0.31 & 0.38 & 0.43 \\
\bottomrule
\end{tabular}
\caption{\textbf{Per-seed parity ablation on CIFAR-10.}
Values are BN-recalibrated test accuracies from validation-selected checkpoints.}
\label{tab:parityseeds}
\end{table}

\subsection{All Paired Comparisons}

Table~\ref{tab:pairedfull} collects the paired comparisons underlying the main
claims.

\begin{table}[t]
\centering
\scriptsize
\setlength{\tabcolsep}{2.1pt}
\begin{tabular}{@{}lrrrr@{}}
\toprule
Paired difference & Mean$\pm$sd & $n$ & $t$ & $p$ \\
\midrule
H1: full $-$ no parity
& $+1.23{\pm}0.43$ & 5 & 6.49 & 0.003 \\

full $-$ bare
& $+4.82{\pm}0.40$ & 5 & 26.92 & $<10^{-4}$ \\

H3: full $1\times$ $-$ bare $4\times$
& $+0.33{\pm}0.49$ & 5 & 1.49 & 0.210 \\

H3$'$: parity $-$ parameter-matched bare
& $+3.09{\pm}0.31$ & 5 & 22.27 & $<10^{-4}$ \\

H4a: progressive $-$ direct40
& $+1.87{\pm}0.22$ & 5 & 19.41 & $<10^{-4}$ \\

H4b: progressive $-$ direct160
& $-0.08{\pm}0.41$ & 3 & $-0.34$ & 0.766 \\

H5: parity gap at $0.25\times$
& $+2.87{\pm}0.67$ & 3 & 7.43 & 0.018 \\

H5: parity gap at $0.50\times$
& $+1.66{\pm}0.40$ & 3 & 7.18 & 0.019 \\
\midrule
full recipe $-$ no EDE
& $-0.38{\pm}0.16$ & 3 & $-4.06$ & 0.056 \\

full recipe $-$ no AT
& $-0.20{\pm}0.22$ & 3 & $-1.58$ & 0.255 \\

full recipe $-$ no EMA
& $+0.02{\pm}0.29$ & 3 & 0.14 & 0.902 \\

full recipe $-$ no diversity
& $-0.07{\pm}0.13$ & 3 & $-1.00$ & 0.423 \\

full recipe $-$ no latent-LR boost
& $-0.14{\pm}0.05$ & 3 & $-4.93$ & 0.039 \\
\bottomrule
\end{tabular}
\caption{\textbf{Complete paired CIFAR-10 statistics.}
All values use BN-recalibrated accuracy from validation-selected checkpoints.
A negative recipe difference means the removal arm scored nominally above the
full recipe.}
\label{tab:pairedfull}
\end{table}

One secondary recipe comparison---removing the latent-weight LR multiplier---
has nominal $p<0.05$, but the difference is only $0.14$ points and is in the
opposite direction to an ablation penalty. With several small-$n$ recipe
comparisons, we do not interpret this isolated value as evidence that removing
the component is intrinsically beneficial. The appropriate conclusion is
simply that none of these auxiliary recipe components is required to obtain
the reported W1A1 accuracy.

\section{Parity Dose Response, Pairing, and Synthetic Tasks}

\subsection{Dose Response}

\begin{table}[t]
\centering
\scriptsize
\setlength{\tabcolsep}{3pt}
\begin{tabular}{@{}lrrl@{}}
\toprule
Configuration & Accuracy (\%) & $n$ & Note \\
\midrule
$|\mathcal{R}|=0$
& \ms{81.74}{0.53} & 3 & no parity \\

$|\mathcal{R}|=1$
& \ms{82.67}{0.10} & 3 & \\

$|\mathcal{R}|=2$
& \ms{82.96}{0.12} & 5 & deployed \\

$|\mathcal{R}|=4$
& \ms{83.47}{0.37} & 3 & \\

$|\mathcal{R}|=8$
& \ms{84.00}{0.08} & 3 & \\
\midrule
random pairs, two planes
& \ms{82.94}{0.10} & 3 & control \\

rolls $(1,7)$
& \ms{83.08}{0.38} & 3 & control \\

rolls $(2,11)$
& \ms{82.76}{0.08} & 3 & control \\
\bottomrule
\end{tabular}
\caption{\textbf{Exact values behind the parity dose-response figure.}
All arms use one uniform direct-to-W1A1 protocol.}
\label{tab:doseapp}
\end{table}

Accuracy rises monotonically across all five tested plane counts. The
$|\mathcal{R}|=2$ arm here
($82.96\pm0.12$, $n=5$) also closely agrees with the independently configured
mechanism-ablation full arm ($83.00\pm0.31$), providing a useful cross-study
consistency check.

The random-pair and alternative-offset controls do not reveal a systematic
advantage over the deployed rolls. We therefore find no evidence, at the
resolution of these experiments, that the circulant wiring restriction imposes
an accuracy penalty. This should not be read as proving that all pairing
patterns are equivalent or that the chosen offsets are globally optimal.

\subsection{Synthetic Degree-2 Tasks}

The main paper reports a compact subset of this experiment to connect
Proposition~1 to measurement. Here we report the complete composition sweep,
including intermediate sign-network widths and the additional signed-readout
control, to distinguish \emph{direct basis access} from representability
obtained indirectly through depth and width.

Each synthetic sample is a 64-dimensional vector in
$\{-1,+1\}^{64}$, and each target depends on seven pairwise interactions.
For the \emph{covered} task, the relevant pair distances lie in the orbit of
$\mathcal{R}=\{1,3\}$; for the \emph{uncovered} task, pair distances are
chosen outside that orbit.

\begin{table}[t]
\centering
\scriptsize
\setlength{\tabcolsep}{2.4pt}
\begin{tabular}{@{}lrrr@{}}
\toprule
Model & Covered & Uncovered & Params. \\
\midrule
Linear FP
& \ms{49.9}{1.5}
& \ms{50.0}{0.3}
& 65 \\

Parity map, $\mathcal{R}=(1,3)$
& \ms{\textbf{100.0}}{0.0}
& \ms{50.2}{0.3}
& \textbf{193} \\

Sign model, one hidden layer, w1024
& \ms{50.0}{0.2}
& \ms{49.9}{1.5}
& 67.6K \\

Sign MLP, two layers, w64
& \ms{62.7}{1.3}
& \ms{61.3}{2.3}
& 8.4K \\

Sign MLP, two layers, w256
& \ms{81.4}{1.1}
& \ms{80.1}{1.7}
& 82.7K \\

Sign MLP, two layers, w1024
& \ms{87.0}{1.9}
& \ms{86.4}{1.1}
& 1.12M \\
\bottomrule
\end{tabular}
\caption{\textbf{Synthetic degree-2 targets.}
Results use three seeds. Covered pairs lie inside the deployed roll orbit;
uncovered pairs lie outside it.}
\label{tab:synth}
\end{table}

The experiment supports both the useful part and the limitation of the theory.
For covered targets, the explicit parity map with only 193 parameters reaches
$100\%$, whereas the FP linear model remains at chance and even the largest
tested two-hidden-layer sign MLP reaches $87.0\%$ with 1.12M parameters.
This does not imply that deeper binary networks cannot represent the target;
rather, it shows the cost of synthesizing an interaction that is available as
a direct coordinate to the parity model.

The uncovered task gives the complementary result. When the relevant pair
distances do not occur in the deployed roll orbit, the same parity map remains
at chance ($50.2\%$). This is exactly the limitation stated in the main-paper
analysis: the implementation restores a structured circulant subset of the
degree-2 Walsh basis rather than all $\binom{n}{2}$ pairwise characters.

All three pre-registered synthetic thresholds are satisfied:
covered $\geq99\%$, FP linear $\leq55\%$, and uncovered $\leq60\%$.
A further variant applies a sign to the parity readout and reaches
$56.6\pm1.3\%$, reinforcing that Proposition~1(i) concerns linear functions
of $\Phi_2$ \emph{before} an additional sign boundary.

\section{Width and Capacity Controls}

\subsection{Capacity Starvation}

\begin{table}[t]
\centering
\scriptsize
\setlength{\tabcolsep}{3pt}
\begin{tabular}{@{}lrrrr@{}}
\toprule
Width & No parity & + parity & Gap & $p$ \\
\midrule
$0.25\times$
& \ms{72.01}{0.40}
& \ms{74.89}{0.28}
& $+2.87$
& 0.018 \\

$0.50\times$
& \ms{78.18}{0.70}
& \ms{79.84}{0.41}
& $+1.66$
& 0.019 \\

$1.00\times$
& \ms{81.77}{0.38}
& \ms{83.00}{0.31}
& $+1.23$
& 0.003 \\
\bottomrule
\end{tabular}
\caption{\textbf{Capacity starvation on CIFAR-10.}
$n=3$ for reduced widths and $n=5$ at full width.}
\label{tab:starveapp}
\end{table}

The paired parity advantage increases monotonically as the base model narrows.
The reduced-width parity/no-parity configurations contain approximately
0.77M/0.39M parameters at $0.25\times$ width and 3.02M/1.51M at
$0.5\times$. This directional result is consistent with the interpretation
that direct pairwise coordinates become more useful when generic composition
capacity is constrained.

\subsection{Equal-Parameter Control}

The parity projection adds learned binary parameters, so an on/off parity
comparison alone cannot distinguish structured capacity from parameter count.
We therefore widen the bare binary model until its parameter count nearly
matches the two-plane parity model.

\begin{table}[t]
\centering
\scriptsize
\setlength{\tabcolsep}{5pt}
\begin{tabular}{@{}lrrr@{}}
\toprule
Configuration & Params. & Accuracy (\%) & $n$ \\
\midrule
Bare, parameter matched
& 11.77M & \ms{79.87}{0.27} & 5 \\

Two-plane parity
& 11.94M & \ms{\textbf{82.96}}{0.12} & 5 \\
\midrule
Difference
& -- & $\mathbf{+3.09{\pm}0.31}$ & 5 \\
\bottomrule
\end{tabular}
\caption{\textbf{Equal-parameter mechanism control.}}
\label{tab:parammatch}
\end{table}

The five paired improvements are
$+3.17$, $+3.26$, $+3.07$, $+3.37$, and $+2.57$ points, producing
$t(4)=22.27$ and $p<10^{-4}$. Thus, at this matched budget, ordinary widening
does not reproduce the benefit of allocating parameters to explicit
degree-2 features.

\subsection{Broader Width Frontier}

The main paper's width analysis draws on three distinct comparisons.

First, under direct training, full \bikan{} at teacher width is compared with
a bare $4\times$ model containing 94.67M parameters. The five paired
differences are
$[0.17,\,1.03,\,-0.27,\,0.15,\,0.55]$ points.
The mean difference is $+0.33\pm0.49$, but $p=0.210$ and the confidence
interval includes zero. We therefore describe the models as statistically
indistinguishable rather than claiming that \bikan{} exceeds the widened
opponent.

Second, the equal-parameter comparison above removes the size confound and
strongly favors parity.

Third, the bare $4\times$ model trained with the complete progressive recipe
reaches $84.76\pm0.02\%$ over its two completed seeds, compared with
$84.38\pm0.14\%$ for teacher-width \bikan{}. This is why the paper's width
claim is explicitly an efficiency claim rather than a dominance claim.

Exploratory full-stack width scaling further shows that parity and width are
complementary rather than mutually exclusive: full-$2\times$ reaches
$84.60\pm0.27\%$ at 47.47M parameters ($n=3$), while full-$4\times$ reaches
85.77\% at 189.32M parameters in the available single run.

\section{Optimization and Recipe Controls}

\subsection{Progressive Training versus Direct W1A1}

The four-stage continuation trajectory uses
W1A8$\rightarrow$W1A4$\rightarrow$W1A2$\rightarrow$W1A1.
A direct 40-epoch W1A1 run matches only the ladder's final-stage budget,
whereas a direct 160-epoch run matches its total optimization steps.

\begin{table}[t]
\centering
\scriptsize
\setlength{\tabcolsep}{4pt}
\begin{tabular}{@{}lrrr@{}}
\toprule
Training & Accuracy (\%) & $n$ & Difference \\
\midrule
Direct W1A1, 40 epochs
& \ms{82.51}{0.25}
& 5
& $-1.87$ \\

Progressive, 160 total
& \ms{\textbf{84.38}}{0.14}
& 5
& -- \\

Direct W1A1, 160 epochs
& \ms{84.43}{0.29}
& 3
& $+0.08$ \\
\bottomrule
\end{tabular}
\caption{\textbf{Optimization-budget decomposition.}}
\label{tab:budgetapp}
\end{table}

Progressive training strongly exceeds the short direct run
($p<10^{-4}$), but not the total-compute-matched direct run
($p=0.766$). We therefore do not attribute an independent accuracy gain to
the precision schedule itself. Its practical role is as a continuation
strategy: a single trajectory produces usable W1A8, W1A4, W1A2, and W1A1
checkpoints before reaching the final binary operating point.

\subsection{Precision Trajectory}

\begin{table*}[t]
\centering
\scriptsize
\setlength{\tabcolsep}{6pt}
\begin{tabular}{@{}lcccc@{}}
\toprule
Stage &
MNIST &
CIFAR-10 &
CIFAR-100 &
Tiny-ImageNet \\
\midrule
FP32 teacher
& \ms{99.53}{0.01}
& \ms{83.11}{0.34}
& \ms{53.24}{2.83}
& \ms{32.95}{1.80} \\

W1A8
& \ms{99.52}{0.03}
& \ms{85.57}{0.19}
& \ms{58.27}{1.71}
& \ms{39.45}{1.56} \\

W1A4
& \ms{99.52}{0.05}
& \ms{85.81}{0.12}
& \ms{58.85}{2.07}
& \ms{39.59}{1.46} \\

W1A2
& \ms{99.51}{0.05}
& \ms{85.51}{0.34}
& \ms{57.55}{1.93}
& \ms{38.75}{1.58} \\

W1A1, BN-recal.
& \ms{99.48}{0.05}
& \ms{84.38}{0.14}
& \ms{55.81}{1.86}
& \ms{37.15}{0.68} \\
\bottomrule
\end{tabular}
\caption{\textbf{Precision trajectory.}
$n=3,5,3,2$ respectively for MNIST, CIFAR-10, CIFAR-100, and
Tiny-ImageNet. Tiny-ImageNet uses the selection fallback described in
Appendix~B.}
\label{tab:precisionapp}
\end{table*}

The intermediate multi-bit states are often more accurate than the final W1A1
model, particularly on CIFAR-10 and CIFAR-100. This observation should not be
interpreted as evidence that low precision intrinsically regularizes KANs;
the experiments were not designed to establish such a mechanism. Rather, the
trajectory shows that progressive training passes through strong executable
intermediate-precision solutions before reaching W1A1, while the
compute-matched control above shows that the schedule itself is not required
for the final W1A1 accuracy.

\subsection{Recipe Ablation}

The three-seed matched full-recipe reference is
$84.35\pm0.17\%$. The corresponding removal arms are:

\begin{table}[t]
\centering
\scriptsize
\setlength{\tabcolsep}{3pt}
\begin{tabular}{@{}lrrr@{}}
\toprule
Configuration & Accuracy & $\Delta$ & $p$ \\
\midrule
Full
& \ms{84.35}{0.17}
& -- & -- \\

No attention transfer
& \ms{84.55}{0.34}
& $+0.20$ & 0.255 \\

No EMA teacher
& \ms{84.33}{0.44}
& $-0.02$ & 0.902 \\

No diversity loss
& \ms{84.42}{0.29}
& $+0.07$ & 0.423 \\

No latent-LR boost
& \ms{84.49}{0.17}
& $+0.14$ & 0.039 \\

No EDE
& \ms{84.73}{0.21}
& $+0.38$ & 0.056 \\
\bottomrule
\end{tabular}
\caption{\textbf{Secondary training-recipe ablation.}
$\Delta$ denotes removal arm minus full.}
\label{tab:recipeapp}
\end{table}

No removal produces a significant accuracy degradation. The no-EDE arm uses
the intended analytic contrast: a mean-center-only weight forward with an
identity straight-through backward, removing Libra standardization and the
EDE backward jointly. These results reinforce the attribution used in the
main paper: the parity mechanism, rather than a uniquely necessary collection
of training heuristics, is the contribution supported by the controlled
experiments.

\section{Preliminary Single-Seed Experiments}

Before the paired validation campaign, several experiments used one seed and
test-based model selection. The best-on-test checkpoint was reloaded at each
stage, and the reported value could use the better of the best-test and
BN-recalibrated score. These results are retained for provenance and historical
comparison only.

\begin{table}[t]
\centering
\scriptsize
\setlength{\tabcolsep}{2.5pt}
\begin{tabular}{@{}lrrrr@{}}
\toprule
Configuration & MNIST & C-10 & C-100 & T-IN \\
\midrule
FP32 teacher
& 99.55 & 83.75 & 54.30 & 31.15 \\

\bikan{} W1A1 best
& 99.50 & 84.93 & 56.92 & 36.84 \\

Hyper-wide W1A1
& 98.31 & 79.51 & 52.51 & -- \\

Hyper-wide FP32 teacher
& 99.52 & 79.79 & 54.54 & -- \\
\bottomrule
\end{tabular}
\caption{\textbf{Preliminary results under the original single-seed,
test-selected protocol.} These values are not pooled with or substituted for
the primary multi-seed results.}
\label{tab:legacy}
\end{table}

The exact parameter counts remain useful for understanding model scale.
The paired-study FP32 teachers contain approximately
1.86M/1.86M/1.98M/2.11M parameters for
MNIST/CIFAR-10/CIFAR-100/Tiny-ImageNet.
The corresponding \bikan{} students contain
11.90M/11.94M/12.68M/13.50M.
The historical hyper-wide W1A1 model contains
94.59M/94.67M/96.14M parameters on the first three datasets, approximately
$7.6$--$8.0\times$ the parameter count of teacher-width \bikan{}.

The historical numbers are not used for inferential claims because their
checkpoint-selection rule differs from the primary protocol.

\section{MNIST Mechanism Ablation}

MNIST provides little power to separate the proposed mechanisms because the
task is saturated. In the preliminary ablation, full and no-parity both reach
99.54\%, frozen STE reaches 99.57\%, bare reaches 99.42\%, and no-RPReLU
reaches 99.27\%. The paired validation campaign therefore allocates its MNIST
budget to the primary precision trajectory rather than repeating mechanism
arms whose expected differences are small relative to the task ceiling.
All mechanism claims in the paper are consequently based on CIFAR-10 rather
than on MNIST.

\section{Tabular and MLP Suite Details}

The main paper now reports the compact performance table for this suite because
it provides useful breadth beyond the convolutional KAGN experiments. This
appendix records the architectures, parameter budgets, recalibration behavior,
and limitations needed to interpret those numbers. All runs use one seed and
the preliminary test-selected protocol; they are not pooled with the
pre-registered multi-seed campaign.

\begin{table*}[t]
\centering
\scriptsize
\setlength{\tabcolsep}{4.2pt}
\begin{tabular}{@{}llrrrr@{}}
\toprule
Dataset & Layer dimensions &
FP32 KAN & \bikan{} & Wide W1A1 & Binary MLP \\
\midrule
Wine
& 13--4--3
& 704 & 1,152 & 2,260 & 141 \\
Dry Bean
& 16--2--7
& 506 & 870 & 1,646 & 126 \\
JSC OpenML
& 16--8--5
& 8,736 & 2,890 & 5,738 & 298 \\
Tiny-ImageNet MLP
& 12288--256--64--200
& 31.75M & 19.10M & 51.64M & 3.22M \\
Traffic
& 72--32--96
& 112.9K & 87.0K & 173.7K & 6.0K \\
\bottomrule
\end{tabular}
\caption{\textbf{Architectures and parameter counts for the exploratory
tabular/MLP suite.} ``Wide'' denotes the widened W1A1 KAN opponent used in
the main-paper breadth table.}
\label{tab:tabular_params}
\end{table*}

\paragraph{Batch-normalization recalibration.}
The compact main-paper table reports the primary stored scores from the
original artifacts. Where recalibration changes the value, the corresponding
BN-recalibrated scores are: Wine, 38.89 for \bikan{} and 52.78 for the wide
model; Dry Bean, 81.23 and 89.31; JSC, 71.62 and 71.39; Tiny-ImageNet MLP,
10.63 and 12.08; and Traffic binary MLP, 0.1611 RMSE. Wine contains only
36 test examples, making its recalibration particularly unstable. These
differences are another reason the suite is presented as exploratory rather
than combined with the validation-selected campaign.

\paragraph{What the suite does and does not show.}
Under the stored single-run scores, \bikan{} exceeds the matched binary MLP on
every task: by 13.89 points on Wine, 30.95 on Dry Bean, 18.24 on JSC, and
3.00 on the Tiny-ImageNet MLP, while reducing Traffic RMSE by 0.0508. This
consistent direction supports the limited claim that a binary KAN
representation retains useful structure relative to a generic binary MLP.

The width comparison is mixed. \bikan{} and the wide W1A1 model are
effectively tied on JSC (71.67 versus 71.76), but the wide model performs
better on Wine, Dry Bean, Tiny-ImageNet MLP, and Traffic. The suite therefore
does not support universal width independence. It is consistent instead with
the controlled CIFAR-10 conclusion: explicit basis structure can be a better
use of a fixed budget, while additional width remains beneficial on some
tasks.

Traffic provides one additional reference. A full-precision ReLU MLP reaches
0.1007 RMSE with 51.5K parameters, essentially matching the FP32 KAN teacher
at 0.1001. The difference emerges after binarization: \bikan{} obtains
0.1082, whereas the binary MLP degrades to 0.1590. This suggests that the
observed advantage is specific to the low-precision representation rather
than an intrinsic superiority of the FP32 KAN on this task.

\section{Cross-Family Historical Context}

These experiments also use the preliminary single-seed protocol.
Teacher-width dense students reach 96.85\% versus a 97.66\% FP32 teacher for
EfficientKAN and 96.79\% versus 97.63\% for PyKAN. The recorded student
parameter count is 818,484 in both artifacts.

The attempted FastKAN family run recorded zero completed epochs and no valid
result, so it is not reported as an empirical data point.

For historical context, earlier FastKAN experiments illustrate how strongly
naive W1A1 can depend on additional capacity or precision: naive W1A1 was
approximately 86.75\%, learnable scales plus distillation recovered 94.42\%,
a widened hybrid design with a binary spline path reached 97.86\% against a
97.87\% teacher, and widened pure W1A1 reached 96.11\%. EfficientKAN
separately reached 97.43\% versus a 97.32\% wide FP32 teacher in a
mixed-precision configuration with higher-precision boundary layers.

Because these experiments differ in protocol and architecture, they are
presented only as context rather than direct competitors to the controlled
CIFAR-10 mechanism studies.

\section{FPGA Implementation and Measurement Provenance}

\subsection{Platform}

All implementations target the Xilinx Zynq-7020
(\texttt{xc7z020-clg400-1}) with a nominal 100\,MHz target and are evaluated
using the MNIST 10,000-image test set. Reported DSP, LUT, FF, and BRAM counts
are taken from Vivado post-place-and-route implementation results rather than
from HLS resource estimates.

Latency has two distinct sources and is never mixed within a claimed speedup.
For the dense GRAM and EfficientKAN kernels, latency is obtained from C/RTL
cosimulation. Full FP32 convolutional RTL simulation is prohibitively slow, so
the KAGN FP32/W1A1 comparison uses the minimum HLS scheduling estimate for
both designs. The convolutional W1A1 RTL simulation that does complete
requires 45,092,425 cycles, corresponding to 450.9\,ms at 100\,MHz, and is
DDR-bound. This value is retained as verification evidence rather than mixed
with the csynth-based FP32 comparison.

\begin{table*}[t]
\centering
\scriptsize
\setlength{\tabcolsep}{4pt}
\begin{tabular}{@{}llrrrrrrcl@{}}
\toprule
Family & Design &
Acc. & DSP & LUT & FF & BRAM & F$_{\max}$ &
Latency & Source \\
\midrule
GRAM
& FP32
& 97.64 & 104 & 14,417 & 13,447 & 30 & 100.7
& 59.6 ms & RTL cosim \\

GRAM
& W1A1
& 97.64 & 6 & 9,365 & 4,791 & 272 & 98.9
& 0.705 ms & RTL cosim \\

GRAM
& W1A1 + Po2-QAT
& 97.61 & \textbf{0} & 9,700 & 5,170 & 272 & 98.9
& 0.705 ms & RTL cosim \\
\midrule
KAGN Conv
& FP32
& 99.55 & 164 & 25,155 & 23,929 & 224 & 103.0
& 401 ms & csynth-min \\

KAGN Conv
& W1A1
& 99.49 & 72 & 9,144 & 11,482 & 215 & 110.0
& 54.8 ms & csynth-min \\

KAGN Conv
& W1A1 + naive Po2
& 98.61 & 12 & 11,555 & 13,746 & 200 & 115.8
& 54.8 ms & csynth-min \\
\midrule
EfficientKAN
& FP32
& 98.22 & 70 & 16,355 & 17,025 & 66 & 106.9
& 54.4 ms & RTL cosim \\

EfficientKAN
& binary configuration
& 96.20 & 54 & 12,081 & 12,680 & 95 & 107.7
& 44.6 ms & RTL cosim \\
\bottomrule
\end{tabular}
\caption{\textbf{Complete post-place-and-route hardware results.}
BRAM counts are 18-Kb blocks. Convolutional latency is the common
\texttt{csynth-min} estimate; dense latency is C/RTL-cosim cycle derived.}
\label{tab:hwmaster}
\end{table*}

\subsection{Interpreting the Hardware Results}

\paragraph{Why two \bikan{} hardware regimes are reported.}
The GRAM and convolutional implementations answer complementary questions.
GRAM isolates the case in which nearly all large learned inner products can be
mapped to XNOR--popcount, exposing the upper end of the W1A1 hardware benefit
and permitting a strictly zero-DSP implementation. The convolutional KAGN
backbone is the architecture used for the primary image-classification
experiments and therefore measures how much of this advantage survives when
normalization, shortcuts, and other non-binary support operations remain.
EfficientKAN provides a further coverage control: because a larger fraction of
its tested datapath remains non-binary, its hardware reductions are
correspondingly smaller.

The benefit of W1A1 depends on how much of the architecture can actually be
mapped to the binary datapath.

For GRAM, the binary transformation covers nearly all large learned inner
products. W1A1 preserves 97.64\% accuracy while reducing DSP usage from 104 to
6 and RTL-cosim latency from 59.6\,ms to 0.705\,ms, an approximately
$84\times$ reduction in cycle-derived latency. BRAM increases because the
binary implementation adopts a different storage strategy and can place more
weights on chip; resource improvements should therefore be read per resource
rather than as uniform reductions in every category.

The convolutional KAGN model retains more non-binary operations around each
binary block. Its W1A1 implementation reduces DSPs from 164 to 72 and LUTs
from 25,155 to 9,144 while preserving accuracy within 0.06 points. Under the
common HLS scheduling model, compute-core latency changes from 401 to
54.8\,ms, a $7.3\times$ reduction. This is not a board-level latency
measurement.

EfficientKAN receives a smaller gain because its tested binary configuration
retains higher-precision boundary layers and a floating-point spline path.
Only part of the accelerator therefore maps to XNOR--popcount, and the
resulting hardware gain is correspondingly smaller.

\subsection{Where the Remaining DSPs Come From}

The binary dot product is
\[
    d
    =
    2\,\operatorname{popcount}
    \left(
        \operatorname{XNOR}(\mathbf{q},\mathbf{b})
    \right)
    -N,
\]
which synthesizes to LUT/FF logic without multiplier DSPs.

For GRAM, the residual DSPs arise primarily from per-channel scale
multiplications and surrounding real-valued operations. For the convolutional
KAGN implementation, additional scale, folded-normalization, shortcut,
RPReLU, head, and indexing arithmetic remain outside the pure XNOR core.
EfficientKAN additionally retains floating-point spline and mixed-precision
arithmetic, so its residual DSP use is intrinsic to a larger fraction of the
tested architecture.

\subsection{Power-of-Two Scale Mapping}

Simply rounding trained GRAM scales to powers of two after training reduces
accuracy to 55.46\%. This failure shows that scale replacement is not a
lossless post-processing step.

Instead, Po2-aware fine-tuning places the rounded power-of-two scale in the
forward pass so that the remaining parameters adapt to the value used by the
hardware. The resulting GRAM model reaches 97.61\%, compared with 97.64\% for
both the FP32 and original W1A1 models, while post-route DSP use falls from six
to exactly zero.

\begin{table}[t]
\centering
\scriptsize
\setlength{\tabcolsep}{3.5pt}
\begin{tabular}{@{}lrrr@{}}
\toprule
Design & DSP & LUT & Accuracy \\
\midrule
GRAM W1A1
& 6 & 9,365 & 97.64 \\

GRAM Po2-QAT
& \textbf{0} & 9,700 & 97.61 \\
\midrule
Conv W1A1
& 72 & 9,144 & 99.49 \\

Conv naive Po2
& \textbf{12} & 11,555 & 98.61 \\
\bottomrule
\end{tabular}
\caption{\textbf{DSP--LUT trade-off from power-of-two scale mapping.}}
\label{tab:po2}
\end{table}

The zero-DSP GRAM design spends only 335 additional LUTs to eliminate the
final six DSPs, approximately 56 LUTs per removed DSP. In the convolutional
model, naive Po2 mapping removes 60 DSPs at a cost of 2,411 LUTs, approximately
40 LUTs per removed DSP.

The convolutional model does not yet achieve a strict zero-DSP implementation
because folded normalization and shortcut scales remain unconstrained during
training. Naively rounding these scales contributes to the observed
$99.49\%\rightarrow98.61\%$ accuracy reduction. A fully Po2-constrained
convolutional training path is therefore left as future work rather than
claimed here.

\subsection{Derived Hardware Reductions}

Table~\ref{tab:hwreductions} converts the absolute implementation results in
Table~\ref{tab:hwmaster} into like-for-like reduction factors. These ratios are
reported only when numerator and denominator use the same measurement source.
In particular, the convolutional latency ratio compares \texttt{csynth-min}
against \texttt{csynth-min}; it is not mixed with the DDR-bound RTL-cosim
number.


\begin{table}[t]
\centering
\scriptsize
\setlength{\tabcolsep}{3.0pt}
\renewcommand{\arraystretch}{1.08}
\begin{tabular}{@{}lcccc@{}}
\toprule
\textbf{Comparison} &
\textbf{DSP} &
\textbf{LUT} &
\textbf{FF} &
\textbf{Lat.} \\
\midrule
GRAM: W1A1 / FP32
& $17.3\times$ & $1.54\times$ & $2.81\times$ & $84.5\times$ \\
KAGN: W1A1 / FP32
& $2.28\times$ & $2.75\times$ & $2.08\times$ & $7.32\times$ \\
EfficientKAN: Bin. / FP32
& $1.30\times$ & $1.35\times$ & $1.34\times$ & $1.22\times$ \\
\bottomrule
\end{tabular}
\caption{\textbf{Derived hardware reductions.}
Reduction factors for DSPs, LUTs, flip-flops (FFs), and latency are computed
only between configurations sharing the same implementation and latency
measurement source.}
\label{tab:hwreductions}
\end{table}

The cross-backbone trend provides a useful control for the hardware
interpretation. GRAM, where almost every large learned inner product maps to
XNOR--popcount, obtains the largest arithmetic and latency reductions. The
convolutional model retains more real-valued support operations and therefore
lands in the middle. EfficientKAN retains higher-precision boundary layers and
floating-point spline operations, so its binary coverage and corresponding
hardware gain are smaller. This pattern is consistent with a coverage-based
hardware explanation rather than with attributing every reduction to the
parity path itself.

\subsection{Verification}

Hardware validation proceeds at several levels:

\begin{enumerate}
    \item exported parameters are replayed in an integer software reference;
    \item the generated C/C++ HLS kernel is compared with that reference;
    \item the exported implementation reproduces the expected predictions on
    all 10,000 MNIST test samples; and
    \item C/RTL cosimulation verifies equivalence between the generated C
    kernel and synthesized RTL.
\end{enumerate}

The completed C/RTL-cosimulation cycle counts are summarized in
Table~\ref{tab:cosimcycles}. These counts are reported as verification and
measurement provenance. The convolutional value is intentionally not used to
form the FP32/W1A1 latency ratio because no corresponding FP32-convolution
RTL-cosim run is available under the same measurement path; all reported
speedup factors therefore compare like-for-like latency sources.

\begin{table}[t]
\centering
\scriptsize
\setlength{\tabcolsep}{4pt}
\begin{tabular}{@{}lrr@{}}
\toprule
Kernel & RTL-cosim cycles & Equivalent at 100 MHz \\
\midrule
GRAM W1A1 & 70,503 & 0.705 ms \\
KAGN Conv W1A1 & 45,092,425 & 450.9 ms \\
EfficientKAN binary & 4,456,054 & 44.6 ms \\
\bottomrule
\end{tabular}
\caption{\textbf{Completed C/RTL-cosimulation runs.}
The KAGN convolutional run is DDR-bound and is used for functional validation,
not for the main FP32/W1A1 speedup.}
\label{tab:cosimcycles}
\end{table}

The Po2-QAT GRAM export passes the same numerical checks while synthesizing
with zero DSPs. These tests establish functional correspondence between the
trained/exported model and the implemented datapath. They do not constitute
physical board-level measurements of energy or end-to-end application
latency.

\paragraph{Model storage example.}
For the dense implementation, the FP32 model contains approximately 0.51M
parameters, or roughly 2.0\,MB when stored as 32-bit values. The
parity-augmented W1A1 model contains approximately 3.26M binary parameters,
corresponding to roughly 0.41\,MB for the 1-bit weights before auxiliary scale
and metadata storage. Thus the parity path can increase the number of learned
binary parameters while still retaining a much smaller raw weight footprint
than the FP32 representation.

\section{Complete Training Details and Reproduction}

This appendix records every setting needed to reproduce the reported numbers
from a clean checkout. Section~B fixed the statistical protocol; this section
fixes the training configuration that protocol was applied to. Unless a table
below states otherwise, every arm of every comparison uses identical values, so
the only variable between arms is the stated treatment.

The supplementary code release contains \texttt{README.md} (per-experiment
commands), \texttt{REPRODUCE.md} (the full campaign, queue files, and analysis
pipeline), and \texttt{PREREGISTRATION.md} (the pre-committed analysis plan).
Every command quoted here is reproduced verbatim in those files.

\subsection{Software and Hardware Environment}

Training used PyTorch~$\geq$~2.0 with torchvision and \texttt{einops} on Python
$\geq$~3.9; the campaign itself ran on CUDA~12.1 wheels. Experiments were
executed on two accelerators sharing one filesystem: a higher-throughput node
(16 CPU cores, \texttt{BKAN\_NUM\_WORKERS=8}, \texttt{OMP\_NUM\_THREADS=4}) and a
second node (8 CPU cores, \texttt{BKAN\_NUM\_WORKERS=4},
\texttt{OMP\_NUM\_THREADS=2}). Dataloader worker counts were held constant per
machine for the entire campaign, because changing them changes the input
pipeline and therefore the comparison.

We do not claim bitwise determinism. GPU floating-point reductions vary with
device, driver version, and cuDNN algorithm selection, so the protocol reports
mean~$\pm$~standard deviation across seeds rather than exact replay. The
launcher exposes a \texttt{--deterministic} flag that additionally requests
deterministic torch algorithms; it is slower and was not used for the reported
numbers. Seeding is applied before any module is constructed, so both weight
initialization and dataloader shuffling are seed-controlled.

Binarized layers deserve a specific caveat here, because they are far more
sensitive to floating-point reassociation than full-precision ones. A change of
one unit in the last place in a pre-activation that happens to lie on a
$\operatorname{sign}$ or rounding boundary flips a discrete decision, so the
layer output moves by $O(1)$ rather than by an ULP. We measured this directly
with the release's bitwise characterization suite: comparing single-forward and
backward snapshots of all 67 covered primitives recorded under two different
PyTorch releases, 56 agreed to within $1.5\times10^{-5}$, while 11 --- every one
of them a binarized or quantized path --- diverged visibly. This is a property
of binarization at thresholds, not a defect, and it does not affect the reported
accuracies, which are averaged over seeds and full training runs rather than
over single forward passes. It does mean that a bitwise characterization
baseline must be re-recorded after a toolchain change rather than carried
across versions.

The FPGA results in Appendix~K additionally require Xilinx Vivado HLS 2019.1.
The code release contains the complete hardware track: the exporters that fold
batch normalization, pack bits and emit the weight and threshold ROMs
(\texttt{bkan/export/}); the generated HLS kernels, self-checking testbenches
and build scripts for all nine designs (\texttt{hls/}); and the export
provenance for each design (\texttt{hw\_data/<design>/manifest.json}, recording
geometry, packing convention, fixed-point scales, source-checkpoint hash, and
replay-versus-PyTorch agreement).

Correctness is verifiable without any Xilinx tooling. The kernels also compile
with a plain C++11 compiler, because the headers fall back to fixed-width
integer types when the arbitrary-precision types are unavailable, and a small
design ships with its test vectors so the export, generation and simulation
chain runs end to end in minutes. Resource and latency figures are the only
part that needs the vendor flow; the measured values are retained in the
release as \texttt{hls/RESULTS.json} so the hardware table can be audited
without rerunning synthesis.

\subsection{Architectures and Parameter Counts}

The FP32 teacher is a degree-3 Gram-polynomial convolutional KAGN at channel
widths $(64,128,256)$. The \bikan{} student uses \emph{the same widths}: this is
the width-free claim, and no student in the primary experiments is wider than
its teacher. The shared macro-architecture is three convolutional blocks with
max-pooling after the first two, adaptive average pooling, and a dense
classification head, with hardtanh between blocks and three intermediate feature
taps used for distillation. Group replication before the shifted quantizer uses
$\text{groups}=8$; the deployed parity offset set is $\mathcal{R}=\{1,3\}$ with
circulant pairing.

\begin{table}[t]
\centering
\scriptsize
\setlength{\tabcolsep}{4pt}
\begin{tabular}{@{}llr@{}}
\toprule
Model & Widths & Parameters \\
\midrule
FP32 KAGN teacher            & $(64,128,256)$   & 1.86\,M \\
\bikan{} student (full)      & $(64,128,256)$   & 11.94\,M \\
Student without parity path  & $(64,128,256)$   & 5.97\,M \\
Bare baseline                & $(64,128,256)$   & 5.97\,M \\
Bare baseline, $4\times$ wide & $(256,512,1024)$ & 94.67\,M \\
Parameter-matched bare       & $(90,180,360)$   & 11.77\,M \\
\bottomrule
\end{tabular}
\caption{\textbf{Model configurations and parameter counts.}
The $4\times$ widths are the conventional hyper-widened binary configuration.
The parameter-matched bare widths are solved numerically by
\texttt{hpc/solve\_bare\_widths.py} and cached so that all five seeds use
identical widths.}
\label{tab:archparams}
\end{table}

The parity path roughly doubles the student's binary parameter count while
adding no hidden channels and no multiply--accumulate operations: each parity
plane is one XNOR per channel fed by a fixed channel rotation, which is pure
wiring in hardware. Appendix~K quantifies the resulting storage footprint.

\subsection{Optimization}

All arms share one optimizer configuration. The only deliberate asymmetry is
the treatment of latent binary weights, which receive zero weight decay and a
boosted learning rate. Decay pulls latent weights toward zero and induces
chronic sign-flip churn late in training, and only the sign is deployed, so
decay buys nothing.

\begin{table}[t]
\centering
\scriptsize
\setlength{\tabcolsep}{4pt}
\begin{tabular}{@{}ll@{}}
\toprule
Setting & Value \\
\midrule
Optimizer                     & AdamW \\
Base learning rate            & $1\times10^{-3}$ \\
Schedule                      & cosine annealing to zero, $T_{\max}=$ stage epochs \\
Latent binary weights         & weight decay $0$, lr $\times\,2.0$ \\
All other parameters          & weight decay $1\times10^{-4}$, lr $\times\,1$ \\
Batch size                    & 128 \\
Teacher optimizer             & AdamW, lr $1\times10^{-3}$, wd $1\times10^{-4}$, cosine \\
Teacher objective             & cross-entropy \\
\bottomrule
\end{tabular}
\caption{\textbf{Optimization settings, identical across every arm.}}
\label{tab:optim}
\end{table}

\subsection{Distillation Objective}

Each student stage minimizes
\[
\mathcal{L}
=
\mathcal{L}_{\mathrm{KD}}
+
\lambda_{\mathrm{AT}}\,\mathcal{L}_{\mathrm{AT}}
+
\lambda_{\mathrm{div}}\,\mathcal{L}_{\mathrm{div}}
+
\lambda_{\mathrm{EMA}}\,\mathcal{L}_{\mathrm{EMA}},
\]
where
$\mathcal{L}_{\mathrm{KD}}
= \alpha T^{2}\,\mathrm{KL}(\sigma(z_s/T)\,\|\,\sigma(z_t/T))
+ (1-\alpha)\,\mathrm{CE}(z_s,y)$
with $T=4.0$ and $\alpha=0.9$.

$\mathcal{L}_{\mathrm{AT}}$ is the scale-free activation-attention criterion of
Zagoruyko and Komodakis applied to the three feature taps: channel-summed
squared activations, $L_2$-normalized over the flattened spatial grid, compared
by mean squared error. Because the student is at teacher width, student and
teacher feature maps match in both channel count and spatial extent, so
\emph{no learnable projectors are required}. This is a direct consequence of
not widening, and it is why the objective is simpler than the two-stage widened
pipelines it replaces.

$\mathcal{L}_{\mathrm{div}}$ is a hinge penalty keeping the grouped sign
thresholds staggered rather than collapsing onto one another.
$\mathcal{L}_{\mathrm{EMA}}$ distills from an exponential-moving-average copy of
the student itself, and is active only in the final W1A1 stage after a warmup.

\begin{table}[t]
\centering
\scriptsize
\setlength{\tabcolsep}{4pt}
\begin{tabular}{@{}llr@{}}
\toprule
Term & Symbol & Value \\
\midrule
KD temperature                 & $T$                       & 4.0 \\
KD blend                       & $\alpha$                  & 0.9 \\
Attention transfer             & $\lambda_{\mathrm{AT}}$   & 1000.0 \\
Shift diversity                & $\lambda_{\mathrm{div}}$  & 0.1 \\
EMA self-teacher               & $\lambda_{\mathrm{EMA}}$  & 0.3 \\
EMA decay                      & --                        & 0.999 \\
EMA warmup (fraction of stage) & --                        & 0.5 \\
\bottomrule
\end{tabular}
\caption{\textbf{Distillation hyperparameters.}
The EMA term is applied in the final W1A1 stage only; the Mode-B dose-response
arms disable it (\texttt{lambda\_ema=0}) so that $|\mathcal{R}|$ is the sole
variable.}
\label{tab:kd}
\end{table}

The large $\lambda_{\mathrm{AT}}$ reflects the scale of the normalized attention
criterion, whose per-element magnitudes are several orders below the logit loss;
it is not a strong-supervision setting. These four auxiliary terms are inherited
from established binary-network practice rather than tuned here, and
Appendix~F's recipe ablation shows that at this budget none of them measurably
improves the W1A1 result: every single-term removal lands within $0.4$ points of
the full recipe, and several are marginally \emph{positive}. They are reported
as second-order settings, not as contributions.

\subsection{Progressive Descent and Post-Training}

The headline recipe steps the activation bit-width through
$8\rightarrow4\rightarrow2\rightarrow1$. Architecture and parameter shapes are
unchanged across stages, so each stage inherits the previous stage's weights
directly. Within each stage the IR-Net error-decay-estimator temperature is
annealed logarithmically from $0.1$ to $10.0$, giving permissive gradients early
and accurate sign gradients late. The anneal restarts per stage.

Two post-training steps precede evaluation. RPReLU slopes are snapped to signed
powers of two, turning the negative-half multiply into a barrel shift for a
multiplier-free implementation. BatchNorm statistics are then re-estimated with
frozen weights over 100 training batches using cumulative averaging. BatchNorm
recalibration matters more for binary networks than is usual, because BatchNorm
effectively sets the sign thresholds and folds into the integer comparison
threshold on hardware; the primary reported metric is measured after this step.

\begin{table}[t]
\centering
\scriptsize
\setlength{\tabcolsep}{4pt}
\begin{tabular}{@{}lrrr@{}}
\toprule
Dataset & Teacher epochs & Epochs/stage & Ablation epochs \\
\midrule
MNIST         & 20 & 20 & 30 \\
CIFAR-10      & 50 & 40 & 40 \\
CIFAR-100     & 60 & 50 & -- \\
Tiny-ImageNet & 60 & 50 & -- \\
\bottomrule
\end{tabular}
\caption{\textbf{Epoch budgets.}
A CIFAR-10 progressive run therefore trains the student for 160 epochs in total
across four stages, which is exactly the budget the compute-matched
\texttt{direct160} control allocates to a single W1A1 stage.}
\label{tab:epochs}
\end{table}

This budget accounting is the reason Appendix~F reports two progressive-versus-
direct comparisons rather than one. At equal \emph{final-stage} budget the
progressive schedule wins by $1.87$ points; at equal \emph{total} optimization
compute the advantage disappears. Both are reported.

\subsection{Data Pipeline and Splits}


\begin{table}[t]
\centering
\scriptsize
\setlength{\tabcolsep}{2.5pt}
\renewcommand{\arraystretch}{1.08}
\begin{tabular}{@{}lll@{}}
\toprule
\textbf{Dataset} & \textbf{Train aug.} & \textbf{Normalization} \\
\midrule
MNIST
& None
& $\mu=(0.1307)$,\quad $\sigma=(0.3081)$ \\

CIFAR-10
& Crop+flip
& \begin{tabular}[c]{@{}l@{}}
  $\mu=(0.4914,0.4822,0.4465)$\\
  $\sigma=(0.2023,0.1994,0.2010)$
  \end{tabular} \\

CIFAR-100
& Crop+flip
& \begin{tabular}[c]{@{}l@{}}
  $\mu=(0.5071,0.4865,0.4409)$\\
  $\sigma=(0.2673,0.2564,0.2762)$
  \end{tabular} \\
\bottomrule
\end{tabular}
\caption{\textbf{Image preprocessing.}
For CIFAR, training uses a $32{\times}32$ random crop with four-pixel padding
and horizontal flipping. Validation and test splits use only the evaluation
transform, without augmentation.}
\label{tab:data}
\end{table}

The validation split is carved from the training split with a \emph{fixed} split
seed of $20260716$, deliberately independent of the experiment seed. Every arm
at every seed therefore sees the identical validation set, so checkpoint
selection cannot introduce a between-arm difference. Primary experiments use
$\texttt{val\_fraction}=0.1$. Tiny-ImageNet uses the $\texttt{val\_fraction}=0$
fallback and selects the final checkpoint instead; as stated in Appendix~B it is
treated as supporting evidence rather than a primary statistical result.

\subsection{Secondary Suites}

The cross-family and tabular studies were conducted under the earlier
single-seed protocol with selection on test, and are reported as exploratory.
Their settings are recorded here for completeness.

\paragraph{Cross-family width-free students (Appendix~J).}
One dense width-free student per KAN family, trained on MNIST at that family's
teacher geometry, distilled from that family's own FP32 teacher with
logits-only KD (dense networks have no spatial attention maps). Teacher: 15
epochs; student: 40 epochs; $\text{groups}=4$; $\mathcal{R}=\{1,3\}$; batch
size~128; lr $1\times10^{-3}$; latent lr multiplier $2.0$; weight decay
$1\times10^{-4}$; $\alpha=0.9$; $T=4.0$;
$\lambda_{\mathrm{div}}=0.1$; BatchNorm recalibration over 100 batches. Student
dimensions are $(784,64,10)$ for the EfficientKAN and PyKAN teachers and
$(784,128,10)$ for the FastKAN teacher.

\paragraph{Tabular and MLP suite (Appendix~I).}
Four arms per dataset under identical budgets: the FP32 KAN teacher; the
width-free \bikan{} student at the teacher's exact layer dimensions; a
\emph{widened} arm using the bare configuration with hidden dimensions scaled by
$4\times$; and a non-KAN arm, a plain binarized MLP at teacher dimensions with
the basis path and parity path disabled and $\text{groups}=1$. Teacher: 60
epochs; binary arms: 80 epochs; batch size~128; lr $1\times10^{-3}$; latent lr
multiplier $2.0$; weight decay $1\times10^{-4}$; $\alpha=0.9$; $T=4.0$;
$\lambda_{\mathrm{div}}=0.1$. Group replication is $4$ for the tabular datasets
and $2$ for the high-dimensional MLP row, whose 12{,}288-dimensional input would
otherwise make the replicated width impractical. The traffic dataset is a
multi-horizon regression task ($72\rightarrow96$) scored by RMSE, with an MSE
distillation term to the teacher's outputs, and additionally trains a
full-precision non-KAN MLP as a reference.

\paragraph{Synthetic Walsh tasks (Appendix~D).}
$n=64$ Boolean coordinates, seven disjoint circular pairs, seeds $0$--$2$,
20{,}000 training and 5{,}000 test examples, 40 epochs, learning rate
$5\times10^{-3}$, batch size~256. The covered regime draws pair distances from
$\mathcal{R}=\{1,3\}$; the uncovered regime draws them from
$\{5,7,11\}$, disjoint from $\mathcal{R}$. Sign-MLP baselines are evaluated at
widths $\{16,64,256,1024\}$ with one and two binary hidden layers. These runs
are CPU-only and complete in minutes.

\subsection{Teacher Sharing and Command Reproduction}

Every arm compared against another at seed $s$ loads \emph{the same} FP32
teacher checkpoint, staged into the run directory before launch. This is not an
optimization: teacher accuracy itself varies across seeds
($83.11\pm0.34$ on CIFAR-10), and pairing removes that variance from the
comparison rather than propagating it. If the cached checkpoint is absent, the
pipeline silently trains a fresh teacher, the arms no longer share one, and the
paired statistics become invalid without raising an error. The queue files
encode the dependency explicitly so that a worker defers a job whose teacher
does not yet exist.

Table~\ref{tab:cmdmap} maps each reported cell to the override that produces it.
All runs share the launcher
\texttt{python run\_seeded.py -e <experiment> --seed <s> --name <tag>
--set val\_fraction=0.1 --teacher <cached teacher>}; only the overrides differ.

\begin{table*}[t]
\centering
\small
\setlength{\tabcolsep}{3pt}
\begin{tabular}{@{}lll@{}}
\toprule
Tag & Experiment & Distinguishing override \\
\midrule
\texttt{abl\_full}      & ablation & \texttt{arms=['full']} \\
\texttt{abl\_no\_parity} & ablation & \texttt{arms=['no\_parity']} \\
\texttt{abl\_bare}      & ablation & \texttt{arms=['bare']} \\
\texttt{prog\_full}     & progressive & (defaults) \\
\texttt{direct40}       & progressive & \texttt{a\_bits\_schedule=(1,)} \\
\texttt{direct160}      & progressive & \texttt{a\_bits\_schedule=(1,)},
                                        \texttt{epochs\_per\_stage=160} \\
\texttt{widbare4}       & width sweep & \texttt{width\_mults=()},
                                        \texttt{include\_bare\_wide=True} \\
\texttt{widfull\_x2}    & width sweep & \texttt{width\_mults=(2.0,)} \\
\texttt{par\_r$k$}      & progressive & \texttt{lambda\_ema=0},
                                        \texttt{parity\_rolls} with $|\mathcal{R}|=k$ \\
\texttt{par\_rand}      & progressive & \texttt{parity\_pairing=random} \\
\texttt{barematch}      & progressive & \texttt{widths=(90,180,360)}, bare config \\
\texttt{w025\_par}      & progressive & \texttt{widths=(16,32,64)} \\
\texttt{w05\_par}       & progressive & \texttt{widths=(32,64,128)} \\
\texttt{rec\_noat}      & progressive & \texttt{lambda\_at=0} \\
\texttt{rec\_noema}     & progressive & \texttt{lambda\_ema=0} \\
\texttt{rec\_nodiv}     & progressive & \texttt{lambda\_div=0} \\
\texttt{rec\_nolat}     & progressive & \texttt{latent\_lr\_mult=1.0} \\
\texttt{rec\_noede}     & progressive & \texttt{estimator=analytic} \\
\bottomrule
\end{tabular}
\caption{\textbf{Run tag to command mapping.}
The dose-response arms additionally set \texttt{a\_bits\_schedule=(1,)} so that
the parity count is isolated from the descent schedule. The complete queue
files, including seeds and teacher dependencies, ship with the code.}
\label{tab:cmdmap}
\end{table*}

\subsection{Artifacts and Analysis}

Every run writes its record incrementally, so the on-disk artifact is complete
up to the last finished epoch even if the job is killed. Each run directory
contains the seed and override record, the cached teacher, the final student
checkpoint, a JSON document with configuration, environment, per-stage results
and parameter counts, and a per-epoch CSV of training loss, test accuracy,
validation accuracy, learning rate, and EDE temperature. Re-running an
experiment never overwrites a previous log; the earlier file is moved aside
under a timestamp.

The primary reported quantity is the BatchNorm-recalibrated test accuracy of the
validation-selected checkpoint. Best-on-test values are also recorded and appear
only where explicitly labeled. Aggregation and the paired $t$-tests of
Appendix~C are produced by a single analysis script over the run directory; the
exact invocations that generated every reported table ship as a script in
\texttt{REPRODUCE.md}, and their outputs are included with the code as
\texttt{results/01\_aggregate\_all.txt} through
\texttt{results/04\_all\_runs.csv}. A reviewer can therefore compare a
replication against the released tables cell by cell without rerunning the
analysis.

The full campaign is 117 jobs and approximately 160--265 GPU-hours. For partial
verification, the release documents three cheaper paths: the synthetic Walsh
tasks reproduce Proposition~1 and its circulant limitation in about ten minutes
on CPU; a single-seed parity ablation reproduces H1's direction in roughly four
GPU-hours; and a three-seed subset of the headline, ablation, and widened arms
reproduces H1, H3, and H4a at reduced statistical power in roughly thirty
GPU-hours.

\end{document}